\documentclass{article}

\PassOptionsToPackage{bookmarks=true}{hyperref} 
\PassOptionsToPackage{round}{natbib}
\usepackage[preprint]{neurips_2026}
\usepackage[utf8]{inputenc}
\usepackage[T1]{fontenc}
\usepackage{amsmath}
\usepackage{amssymb}
\usepackage{amsthm}
\usepackage{natbib}
\usepackage{subcaption}
\usepackage{caption}
\usepackage{algorithm2e}
\RestyleAlgo{ruled}
\usepackage{dsfont}
\usepackage{float}
\usepackage{comment}
\usepackage{graphicx}
\usepackage{booktabs}
\usepackage{multirow}

\theoremstyle{remark}
\newtheorem{remark}{Remark}
\theoremstyle{plain}
\newtheorem{lemma}{Lemma}

\DeclareMathOperator*{\argminC}{\arg\min}

\title{Change Detection in Probability Flow ODE:\\ Online Testing in Diffusion Latent Spaces}

\author{%
  Artem Kraevskiy \\
  National Research University Higher School of Economics \\
  Moscow, Russia \\
  \texttt{aakraevskiy@hse.ru} \\
  \And
  Artem Prokhorov \\
  The University of Sydney Business School \& CIREQ \& CEBDA \\ Sydney, Australia \\
  \texttt{artem.prokhorov@sydney.edu.au}
}

\begin{document}

\maketitle

\begin{abstract}
A rapidly growing range of sequential data tasks, such as identifying trend reversals in financial markets, auto-segmenting video and audio recordings, detecting changes in movement direction from motion sensors cannot be fully addressed without detection of distributional shifts in time-ordered data. We consider a sequential change-point detection problem where the conditional density  switches at an unknown time, yet neither the pre- nor post-change distribution admits a closed-form. Classical likelihood-ratio statistics are inapplicable in this settings.

A conditional diffusion model, trained on pre-change-point data with a frozen context encoder, defines a deterministic bijection via the probability flow ODE. Pre-change observations are mapped onto standard Gaussian latent variables. Post-change observations, processed through the same frozen map, deviate from this reference. We employ the Maximum Mean Discrepancy as the test statistic, derive closed-form expressions for its components under the Gaussian null, and establish its asymptotic distribution as a degenerate U-statistic. Afterwards we apply an online detection procedure of Shiryaev--Roberts to the resulting statistic with exact threshold calibration.

The method detects arbitrary distributional shifts, including covariance rotations and higher-order structural breaks, without parametric assumptions on either regime.
\end{abstract}


\newpage
\section{Problem Formulation}
\label{sec:problem}

Let $(X_t)_{t \geq 1}$, $X_t \in \mathbb{R}^d$, be an observed sequence and let
$\mathcal{F}_t = \sigma(X_1,\ldots,X_t)$ denote its natural filtration.
Fix a look-back window $l \geq 1$ and let $\mathbf{h}_{t-1} \in \mathbb{R}^m$
be a summary of the history $X_{t-l},\ldots,X_{t-1}$ produced by a causal
encoder (e.g.\ GRU, Transformer encoder, MLP, etc.), with $\mathbf{h}_0 = 0$.

Let $\theta \in \mathbb{N}$ be an unobservable change-point, independent of
$\mathcal{F}_\infty$ and distributed geometrically,
$\mathbb{P}(\theta = k) = (1-\rho)^{k-1}\rho$, $\rho \in (0,1)$.
The conditional distribution of $X_t$ given $\mathcal{F}_{t-1}$ is
\begin{equation}
  X_t \mid \mathcal{F}_{t-1}
  \;\sim\;
  \begin{cases}
    p_0(\,\cdot \mid \mathbf{h}_{t-1}) & t < \theta,\\[2pt]
    p_1(\,\cdot \mid \mathbf{h}_{t-1}) & t \geq \theta,
  \end{cases}
\end{equation}
where $p_0(\cdot\mid\mathbf{h})$ and $p_1(\cdot\mid\mathbf{h})$ are families of
absolutely continuous conditional densities with $p_0 \neq p_1$.
Neither density is assumed to have a known analytic form.

\cite{Shiryaev} and \cite{Roberts} showed that, for the Bayes risk
$$\mathcal{R}(\tau) = c\,\mathbb{E}[(\tau-\theta)^+] + \mathbb{P}(\tau<\theta), \qquad c>0,$$
trading off detection delay against the false-alarm probability, the rule minimising $\mathcal{R}(\tau)$ over stopping times $\tau$ has the form:
$\tau^* = \inf\{t : \pi_t \geq \pi^*\}$, where
$\pi_t = \mathbb{P}(\theta \leq t \mid \mathcal{F}_t)$ is the posterior change-point probability and $\pi^*$ is determined by a free-boundary problem (\cite{Shiryaev_Peskir}).
Computing $\pi_t$ requires the conditional log-likelihood ratio
$\log \bigg[\dfrac{p_1(X_t\mid\mathbf{h}_{t-1})}{p_0(X_t\mid\mathbf{h}_{t-1})} \bigg]$,
which is intractable whenever the densities $p_0$ and $p_1$ are unknown.

Another crucial factor that makes the direct application of the Shiryaev-Roberts approach impractical is the curse of dimensionality. Classical approaches to multivariate PDF estimation (kNN, tree-based algorithms, KDE, etc.) suffer even when the dimensionality is moderately large. Furthermore, it may be difficult to consider various domain-specific aspects for empirical distributions such as heavy-tails and tail-dependence in financial time series or sparsity of images in video frame sequences.

We address this problem by applying a trained Probability Flow ODE (PF-ODE) generative model~\citep{SongEtAl2021SDE} -- trained by denoising score matching against the conditional Stein score (Appendix~\ref{app:score-training}) -- as a conditional encoder that, provided the model is well trained, deterministically maps\footnote{The stochastic reverse process of models such as DDPM~\citep{HoEtAl2020DDPM} is a Markov kernel, not a deterministic bijection, and therefore cannot serve this role; see Appendix~\ref{app:bijection} for why injectivity of the encoder is essential to the argument below.} a complex, unknown multivariate PDF in the original data domain to a latent space following $\mathcal{N}(0, I)$, and back again. If the current data $X_t$ comes from the same distribution as its previous values (aggregated into the embedding $\mathbf{h}_{t-1}$) then the corresponding latent variable $z_t$ will follow this standard normal distribution. If the distribution has changed,  the projection into latent space will no longer follow $\mathcal{N}(0, I)$.

In this sense, we want to test the hypothesis that the distribution of latent variables is $\mathcal{N}(0, I)$. The moment when we reject this hypothesis would be considered a change-point. 

The paper proceeds as follows. First, we derive a formula for maximum mean discrepancy (MMD) against a Gaussian alternative and find its asymptotic distribution. Second, we apply the Shiryaev-Roberts approach and describe an algorithm for finding the optimal threshold $\pi^{*}$. Finally, we apply the proposed approach to synthetic and real-world data.

\section{Maximum Mean Discrepancy}
\label{sec:mmd}

\subsection{Definition and RKHS characterisation}

Let $k : \mathbb{R}^d \times \mathbb{R}^d \to \mathbb{R}$ be a symmetric positive-definite
kernel and $\mathcal{H}_k$ the associated reproducing kernel Hilbert space (RKHS) with
inner product $\langle \cdot, \cdot \rangle_{\mathcal{H}}$.
For a distribution $P$ on $\mathbb{R}^d$, the Riesz representation theorem~\citep{Aronszajn1950,SmolaEtAl2007}
in $\mathcal{H}_k$ guarantees a unique element $\mu_P \in \mathcal{H}_k$
(the kernel mean embedding), satisfying:
$$\mathbb{E}_{X \sim P}[f(X)] = \langle f,\,\mu_P \rangle_{\mathcal{H}},
\qquad \mu_P(\cdot) = \mathbb{E}_{X \sim P}[k(\cdot, X)]$$

The \emph{Maximum Mean Discrepancy} (MMD) between two distributions $P$ and $Q$ is
the distance between their mean embeddings in $\mathcal{H}_k$:
$$\mathrm{MMD}(P, Q;\, k) = \sup_{\|f\|_{\mathcal{H}} \leq 1}
  \bigl(\mathbb{E}_{X \sim P}[f(X)] - \mathbb{E}_{Y \sim Q}[f(Y)]\bigr)
= \|\mu_P - \mu_Q\|_{\mathcal{H}}$$

Expanding the squared norm using the reproducing property $\langle k(\cdot,x), k(\cdot,y)\rangle_\mathcal{H} = k(x,y)$:
$$\mathrm{MMD}^2(P, Q) = \mathbb{E}_{X,X' \sim P}[k(X, X')]
- 2\,\mathbb{E}_{\substack{X \sim P \\ Y \sim Q}}[k(X, Y)]
+ \mathbb{E}_{Y,Y' \sim Q}[k(Y, Y')]$$

A kernel $k$ is called characteristic if $\mu_P = \mu_Q \Rightarrow P = Q$,
so that $\mathrm{MMD}(P,Q) = 0 \Leftrightarrow P = Q$ and MMD is a proper metric on the space of probability measures.

\subsection{Estimator against a standard normal reference distribution}

In our setting, the reference distribution $Q = \mathcal{N}(0, I)$ is known exactly.
Given a window of encoded observations $Z_1, \ldots, Z_w \in \mathbb{R}^d$, we estimate $\mathrm{MMD}^2(\hat{P}_w, \mathcal{N}(0,I))$ by decomposing it into three expectations:
$$\widehat{\mathrm{MMD}}^2 =
  \underbrace{\frac{1}{w^2}\sum_{i=1}^{w}\sum_{j=1}^{w} k(Z_i, Z_j)}_{A}
  - \underbrace{\frac{2}{w}\sum_{i=1}^{w} \mathbb{E}_{Y \sim Q}[k(Z_i, Y)]}_{ B}
  + \underbrace{\mathbb{E}_{Y,Y' \sim Q}[k(Y, Y')]}_{C}$$

This estimator includes the diagonal $k(Z_i, Z_i)$ in Term A, making it biased. Write $\mathbb{E}_0[\cdot]$ for expectation under $H_0$, i.e.\ with $Z_i \overset{\text{iid}}{\sim} P_0 = Q$; then $$\mathbb{E}_0\bigl[\widehat{\mathrm{MMD}}^2\bigr] = \frac{k_0 - C}{w}, \qquad k_0 = k(z,z)$$ which is strictly positive whenever $k_0 > C$.\footnote{Derivation of this bias is given in Appendix~\ref{app:BC-derivation}, ``Bias of the diagonal-included estimator.''} The unbiased U-statistic estimator is the one with the removed diagonal elements: $$\widehat{\mathrm{MMD}}^2_u = \frac{1}{w(w-1)}\sum_{i \neq j} k(Z_i, Z_j)- \frac{2}{w}\sum_{i=1}^{w}\mathbb{E}_Y[k(Z_i, Y)] + C,$$
where $ \mathbb{E}_0\bigl[\widehat{\mathrm{MMD}}^2_u\bigr] = 0$.\footnote{The proof that this expectation is exactly zero is given in Appendix~\ref{app:mmd-asymptotics}, Step 1.}

We will be using the Radial Basis Function (RBF) kernel $k(x,y) = \exp(-\|x-y\|^2/2\sigma^2)$ as it is characteristic on $\mathbb{R}^d$
for any $\sigma > 0$ \citep{SriperumbudurEtAl2010}.

For $k(x,y) = \exp\!\bigl(-\|x-y\|^2/2\sigma^2\bigr)$ and $Q = \mathcal{N}(0,I)$, both terms reduce to products of one-dimensional Gaussian integrals, since the kernel factorises over coordinates and $Q$ is Gaussian.\footnote{Term $C$ follows from the moment generating function of $\|Y-Y'\|^2$ under $Y-Y'\sim\mathcal{N}(0,2I)$; Term $B(z)$ follows from completing the square in the cross term $(z_l-y)^2$ before integrating out $y$. The full derivation of both is given in Appendix~\ref{app:BC-derivation}.} 

$${C = \left(\frac{\sigma^2}{\sigma^2 + 2}\right)^{d/2}}
\qquad\qquad
{B(z) = \left(\frac{\sigma^2}{\sigma^2 + 1}\right)^{d/2}
  \exp\!\left(-\frac{\|z\|^2}{2(\sigma^2 + 1)}\right)}$$

\subsubsection*{Assembled estimator}

Substituting the results into the formula, the MMD² against $\mathcal{N}(0,I)$ is:
$$\widehat{\mathrm{MMD}}^2 =
\frac{1}{w^2}\sum_{i,j} e^{-\|Z_i-Z_j\|^2/2\sigma^2}
\;-\; \frac{2}{w}\left(\frac{\sigma^2}{\sigma^2+1}\right)^{d/2}
  \!\sum_{i=1}^{w} e^{-\|Z_i\|^2/2(\sigma^2+1)}
\;+\; \left(\frac{\sigma^2}{\sigma^2+2}\right)^{d/2}$$

Only Term A requires pairwise distances and scales as $O(w^2 d)$; Terms B and C
are $O(wd)$ and $O(1)$ respectively. Under $H_0$, the finite-sample bias is:
$$\mathbb{E}_0\bigl[\widehat{\mathrm{MMD}}^2\bigr]
= \frac{1-C}{w} = \frac{1}{w}\!\left[1 - \left(\frac{\sigma^2}{\sigma^2+2}\right)^{d/2}\right]
\;\xrightarrow{w\to\infty}\; 0$$

\subsection{Asymptotic null distribution}
\label{sec:null_dist}

Under $H_0$ ($Z \sim P_0 = Q$), the centred kernel $\tilde{k}(x,y) = k(x,y) - B(x) - B(y) + C$ has zero conditional mean and $\mathbb{E}_{Z'\sim P_0}[\tilde k(z,Z')] = 0$ for all $z$, so the unbiased estimator $\widehat{\mathrm{MMD}}^2_u$ is a first-order degenerate U-statistic: the usual $\sqrt{w}$-CLT term vanishes, and the correct normalisation is $w$, not $\sqrt{w}$.\footnote{Verification of the degeneracy condition is given in Appendix~\ref{app:mmd-asymptotics}, Step 1.} Expanding $\tilde k$ via Mercer's theorem, $\tilde k(x,y) = \sum_l \lambda_l\,\phi_l(x)\phi_l(y)$, the rescaled statistic converges to a weighted sum of centred $\chi^2_1$ variables (\citealt{Hall1984}; \citealt[Thm.~12]{GrettonEtAl2012}):\footnote{The spectral decomposition and the mode-by-mode CLT argument are given in Appendix~\ref{app:mmd-asymptotics}, Step 2.}

$${w\cdot\widehat{\mathrm{MMD}}^2_u \xrightarrow{d}
\sum_{l=1}^{\infty} \lambda_l\,(Z_l^2 - 1),
\qquad Z_l \overset{\mathrm{iid}}{\sim} \mathcal{N}(0,1)}.$$

A generalised $\chi^2$ law has no closed-form CDF and, for the RBF kernel under $P_0 = \mathcal{N}(0,I)$, the eigenfunctions are weighted Hermite functions with geometrically decaying eigenvalues, $\lambda_m \propto r^m$, where $r = \dfrac{\sigma^2+2-\sigma\sqrt{\sigma^2+4}}{2} \in (0,1)$, so the limit is dominated by the first few modes.\footnote{Derivation of the RBF eigenvalues is given in Appendix~\ref{app:mmd-asymptotics}, Step 3, which also tabulates $C$, $B$, and the resulting null limit for the linear, polynomial, Laplace and IMQ kernels (Table~\ref{tab:kernels}).} 

The choice of kernel determines which distributional changes are detectable. The linear kernel sees only mean shifts, degree-$p$ polynomials see moments up to order $p$, while RBF, Laplace and Inverse Multiquadric (IMQ) are characteristic and detect any $P \neq Q$ (Table~\ref{tab:kernels}).

\subsection{Asymptotic alternative distribution}

Under the alternative ($Z \sim P_1 \neq Q$, with $\delta^2 := \mathrm{MMD}^2(P_1,Q) > 0$), degeneracy breaks down. The \citet{Hoeffding1948} decomposition of $\tilde k$ has a first-order term $h_1(z) = \mathbb{E}_{Z'\sim P_1}[\tilde k(z,Z')] - \delta^2$ that is identically zero under $H_0$ but  nonzero under $H_1$ in general, since $\mathbb{E}_{Z'\sim P_1}[k(z,Z')] \neq B(z)$ whenever $P_1 \neq Q$.\footnote{The full Hoeffding decomposition, and the verification that $h_1 \equiv 0$ under $H_0$ but not under $H_1$, are given in Appendix~\ref{app:mmd-asymptotics}, Alternative case.} The resulting linear term is a sum of $w$ i.i.d.\ mean-zero contributions, $O_p(w^{-1/2})$ by the ordinary CLT, while the remaining degenerate second-order term is $O_p(w^{-1})$ and asymptotically negligible by Slutsky's theorem.\footnote{The order comparison between the linear and degenerate terms is given in Appendix~\ref{app:mmd-asymptotics}, Alternative case.} The estimator therefore reverts to standard $\sqrt{w}$-rate asymptotic normality around the true value:
$$\sqrt{w}\!\left(\widehat{\mathrm{MMD}}^2_u - \delta^2\right) \xrightarrow{d} \mathcal{N}\!\left(0,4\sigma_1^2\right),
\qquad \sigma_1^2 = \mathrm{Var}_{Z\sim P_1}\!\left[\mathbb{E}_{Z'\sim P_1}[\tilde{k}(Z,Z')]\right]$$
Rescaling to the same $w$-normalisation used for the null statistic gives us the following result:

$${w\cdot\widehat{\mathrm{MMD}}^2_u \sim 
\mathcal N\!\left(w\delta^2, 4w\sigma_1^2\right)}$$

a mean growing linearly in $w$ against the null's $O_p(1)$ generalised-$\chi^2$ behaviour (Appendix~\ref{app:mmd-asymptotics}, Step 2). We use this distinct behaviour to force detection power to go to one as the window grows.

\section{Shiryaev-Roberts Statistic for MMD-Based Change-Point Detection}
\subsection{The Shiryaev-Roberts statistic}

Consider first the classical change-point detection setting. Let  $X_1, X_2, \ldots$ denote an i.i.d\ sequence from a known density $p_0$ before an unknown shift moment $\theta$ and from a known density $p_1 \neq p_0$ after $\theta$. The Shiryaev-Roberts (SR) statistic \citep{Shiryaev,Roberts} is defined by the recursion $$R_t = (1 + R_{t-1})\,\Lambda_t, \qquad R_0 = 0, \qquad \Lambda_t = \frac{p_1(X_t)}{p_0(X_t)},$$ which has the equivalent closed form
$$R_t =\sum_{k=1}^{t} \frac{p_1(X_k)\cdots p_1(X_t)}{p_0(X_k)\cdots p_0(X_t)}= \sum_{k=1}^{t}\prod_{j=k}^{t} \Lambda_j.$$
Each term is the likelihood ratio for the hypothesis``the change occurred at $k$'' against ``no change observed,'' so $R_t$ aggregates evidence for a change having occurred at any past time $k \leq t$, rather than committing to a single candidate $k$ as CUSUM does.

\subsection{The Shiryaev-Roberts recursion for mixture MMD likelihood ratio.}

Let
$$\tilde{S}_t \;:=\; \widehat{\mathrm{MMD}}^2_t - \mu_0,
\qquad
\mu_0 := \mathbb{E}_0\bigl[\widehat{\mathrm{MMD}}^2\bigr] = \frac{k_0-C}{w},$$
denote the bias-corrected window summary statistic and $\widehat{\mathrm{MMD}}^2_t$ the closed-form RBF estimator of Section~\ref{sec:mmd}.  $\widehat{\mathrm{MMD}}^2_t$ is computed on the $t$-th non-overlapping window of $w$ consecutive encoded observations $Z_{(t-1)w+1},\ldots,Z_{tw}$, and $\mu_0$ is its exact finite-sample null bias. Its null and alternative densities are the ones derived above.

By the independence result for the frozen-context, non-overlapping-window construction actually used online,\footnote{See Appendix~\ref{app:algorithm}, ``Design choices,'' for the argument.} the sequence $\tilde{S}_1, \tilde{S}_2, \ldots$ is i.i.d.\ within each regime. Its null density $p_0(\cdot) =: g_0(\cdot)$ is the density of $G_0$, the exact (finite-$w$) law of $\tilde{S}_t$ under $H_0$: as $w\to\infty$, $w\tilde{S}_t$ converges to the generalised-$\chi^2$ law $\sum_{l=1}^\infty\lambda_l(Z_l^2-1)$ of Section~\ref{sec:null_dist}, so $G_0$ is the finite-sample analogue of that law and similarly has no closed form. Thus, $g_0$ is evaluated numerically by Monte Carlo simulation of $\tilde{S}_t$ under the known null $\mathcal N(0,I)$ (Appendix~\ref{app:numerical-stability}). 

For its alternative density, we can average over the unknown signal
strength $\delta^2$ using the Tartakovsky-Spivak
exponential prior \citep{TartakovskySpivak2022,TartakovskyNikiforovBasseville2014}. This gives 
$\bar{p}_1$. 
Replacing
$\Lambda_t$ with $\Lambda_t^{\pi} = \bar{p}_1(\tilde{S}_t)/g_0(\tilde{S}_t)$ for
the generic, unavailable ratio $p_1(X_t)/p_0(X_t)$ gives
$$\Lambda_t^{\pi} =
\frac{\alpha\,
\exp\!\left(-\alpha\tilde{S}_t + \frac{1}{2}\alpha^2 v_1^2\right)
\Phi\!\left(\dfrac{\tilde{S}_t - \alpha v_1^2}{v_1}\right)}
{g_0(\tilde{S}_t)},$$
where the SR statistic is now indexed by windows rather than individual observations, with the equivalent form of $R_t$
carried over unchanged:

$${R_t = \bigl(1 + R_{t-1}\bigr)\,\Lambda_t^{\pi}, \qquad R_0 = 0,
\qquad \tau_A = \inf\{t \geq 1 : R_t \geq A\}}.$$

Replacing the true, unavailable likelihood ratio by
its Bayes average over $\delta^2$ turns $R_t$ into a generalised
likelihood-ratio SR statistic in the sense of
\citet{TartakovskyVeeravalli2005}. 
No online estimate of $\delta^2$ enters $\Lambda_t^\pi$ 
though the prior hyperparameters
$(\alpha, v_1)$ still have to be set.  

\subsection{Choosing $\alpha$ and $v_1$.}
There are two main approaches to hyperparameter selection. The first needs the researcher's decision. Under the exponential prior of Tartakovsky-Spivak, $\mathbb{E}_\pi[\delta^2] = 1/\alpha$. Thus, mixing $\alpha$ is a minimum-detectable-effect choice, that is the
smallest post-change signal that the mixture is tuned for. It follows the same logic as choosing an effect size in classical power analysis. 
The second method, used in Section~\ref{sec:experiments}, replaces this judgement with a plug-in estimate of $(\alpha, v_1)$ from a window of data. 
The construction and its limitations are
given in Appendix~\ref{app:pilot-calibration}. 

The two approaches are not mutually exclusive and 
and neither affects the false-alarm rate, which is governed entirely by the threshold calibration. 
A poorly chosen $(\alpha, v_1)$ only makes the detector slower to react to whatever shift eventually occurs.

\subsection{Threshold calibration}
\label{sec:calibration}

We consider two optimality criteria 
for the SR statistic $R_t$. 
One requires committing to an explicit cost trade-off between delay and false alarm and delivers an exactly optimal threshold for that trade-off, the other requires only an operational false-alarm budget and delivers an asymptotically optimal threshold. 

\paragraph{Approach A -- Bayes risk minimisation (Shiryaev--Peskir).}
Fixing a cost $c>0$ for the trade-off between delay and false alarms, the Bayes-type
risk of Section~\ref{sec:problem}, $\mathcal{R}(\tau) = c\,\mathbb{E}[(\tau-\theta)^+]
+ \mathbb{P}(\tau<\theta)$, can be expressed in terms of $R_n$ as 
$\Pi_n = R_n/(1+R_n) \in [0,1)$. This is the window-indexed analogue of
$\pi_t=\varphi_t/(1+\varphi_t)$ used throughout \citet{Shiryaev_Peskir}. It can be used an
optimal stopping problem for the Markov chain $(\Pi_n)_{n\geq0}$. The problem becomes
$$V(x) = \inf_{n \geq 0} \mathbb{E}_x\!\left[1-\Pi_n + c\sum_{k=0}^{n-1}\Pi_k\right],$$
and can be solved using their Wiener and Poisson disorder problems (\S22, \S24).\footnote{Full derivation in Appendix~\ref{app:calibration}, Approach A.}

This yields, for a given cost $c$, the threshold $b^{*}(c)$ 
that is exactly optimal for that specific trade-off, together with the minimal achievable risk $V$. For example, it traces out the
best possible delay/false-alarm frontier as $c$ varies. 

\paragraph{Approach B -- ARL-constrained calibration (Pollak).}
Instread we can fix a target average run length $\gamma$, i.e., how many windows between false alarms we can tolerate. This requires no assumption
about how often changes occur, therefore, calibration reduces to solving:
$$\mathbb{E}_\infty[\tau_A] = \gamma$$
for $A$.
\footnote{Full derivation in Appendix~\ref{app:calibration},
Approach B.} 

This produces the threshold that gives the desired
false-alarm rate, and by the \citet{Pollak1985} theorem, the resulting threshold-crossing rule is asymptotically minimax-optimal as $\gamma\to\infty$. No other rule with the same false-alarm rate can have smaller
worst-case delay. 

Approach A is valuable as a theoretical benchmark, 
in our experiments we use Approach B. It
requires only a false-alarm budget $\gamma$ that practitioners can state directly, 
so no separate optimisation over $c$ is needed. 

\subsection{The full detection procedure}
The components described above (the frozen PF-ODE encoder, the closed-form RBF MMD$^2$, the Tartakovsky-Spivak mixture likelihood ratio, and the Pollak-calibrated SR threshold) combine into an online procedure with a one-time offline stage and a per-window online iteration. Appendix~\ref{app:algorithm} contains the full pseudocode in Algorithm~\ref{alg:sr-detection} and a discussion of the design choices.

\section{Experiments on synthetic data}
\label{sec:experiments}

We demonstrate the procedure 
on three synthetic $p_0\!\to\!p_1$
pairs in $\mathbb R^2$, chosen to reflect 
different kinds of distributional change while holding window $w=25$, RBF bandwidth $\sigma=\sqrt2$, 
and a false-alarm budget $P_{\rm FA}\le5\%$ over a $175$-window monitoring horizon, calibrated by Approach B (Appendix~\ref{app:calibration}).

For each pair, a GRU history encoder of ~\cite{ChoEtAl2014GRU} paired with a pre-norm-residual denoiser that uses FiLM-conditioning of \cite{PerezEtAl2018FiLM} and v-prediction of \cite{SalimansHo2022} (see 
Appendix~\ref{app:architectures} for full specification) is trained from scratch for $3{,}000$ steps. The training uses the Adam optimiser \citep{KingmaBa2015Adam} with an EMA of the weights on series of length $256$. 
The context is then frozen from a $p_0$-only burn-in,
$(\hat\alpha,\hat v_1)$ calibrated as in Appendix~\ref{app:pilot-calibration},
and the threshold calibrated as above. Each trained detector is finally run
once on a fresh test series of length $L=200$ with a genuine, previously
unseen change-point at $\tau=100$.

\begin{description}
\item[Gaussian mixture rotation.]\mbox{}\\[-6pt]
\begin{itemize}
    \item $p_0=\tfrac12\mathcal N([-2,0],I)+\tfrac12\mathcal N([2,0],I)$.
    \item $p_1$ is the same two-component mixture rotated $90^\circ$. Mean
    and per-component covariance are unchanged (only the
    joint arrangement of the two modes rotates).
\end{itemize}

\item[Four clusters $\to$ one cluster.]\mbox{}\\[-6pt]
\begin{itemize}
    \item $p_0$ places four Gaussian blobs ($\mathrm{std}=0.35$) at the corners $(\pm2,\pm2)$.
    \item $p_1$ collapses all mass into a single centred blob ($\mathrm{std}=1.2$). A change in the number of modes rather than in any single low-order moment.
\end{itemize}

\item[Blob $\to$ Ring.]\mbox{}\\[-6pt]
\begin{itemize}
    \item $p_0=\mathcal N(0,I)$.
    \item $p_1$ is a thin ring at radius $2.5$ ($\mathrm{std}=0.2$). A
    purely radial, rotationally-symmetric redistribution of mass away from
    the centre, with near-identical mean.
\end{itemize}
\end{description}

\begin{figure}[h]
\centering
\includegraphics[width=\textwidth]{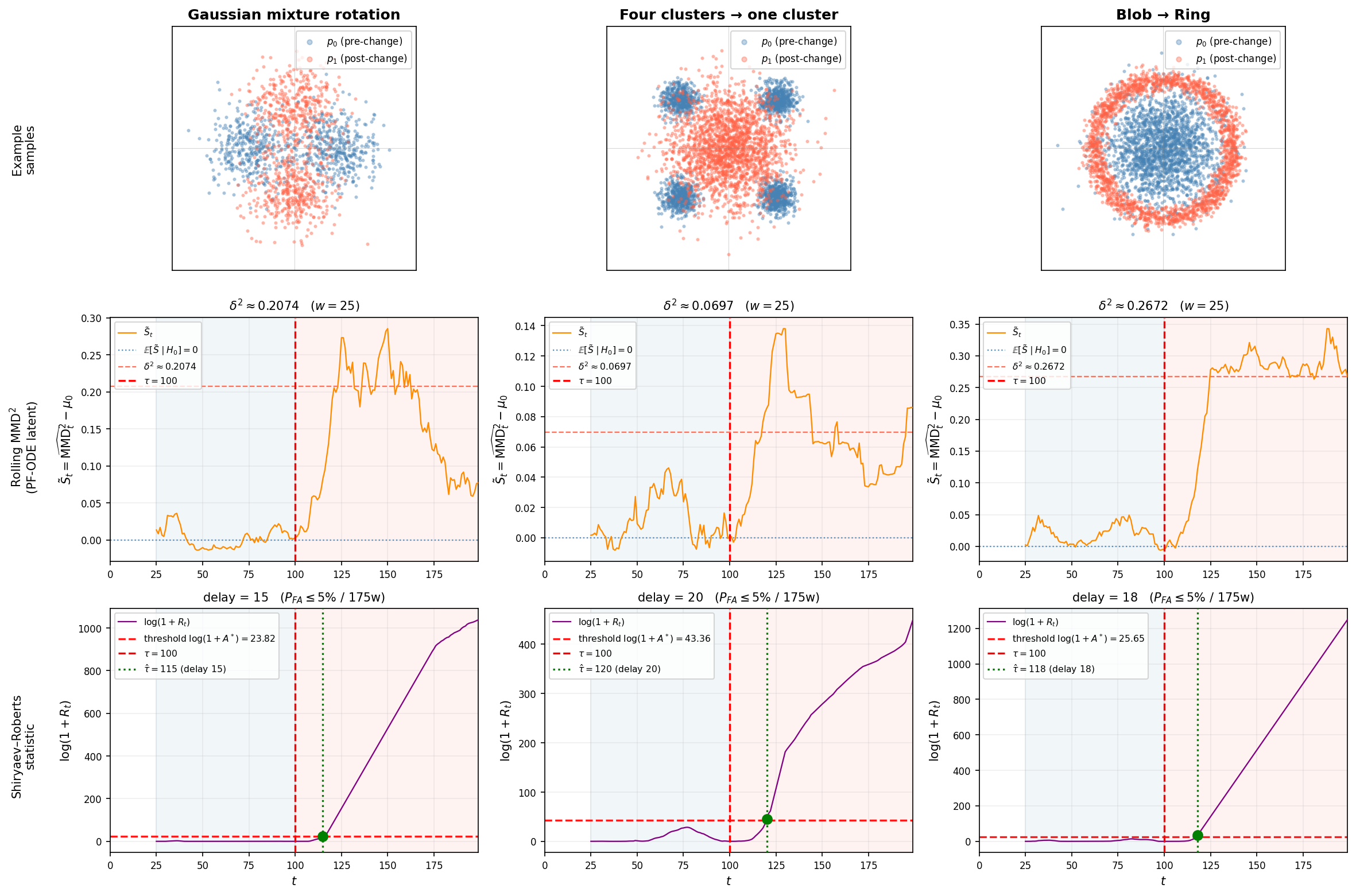}
\caption{Top: example draws from $p_0$ and $p_1$ for each pair. Middle:
the rolling $\widehat{\mathrm{MMD}}^2_t-\mu_0$ statistic, with
its $\hat\delta^2$ under $H_1$ marked. Bottom: the
Shiryaev--Roberts statistic $\log(1+R_t)$, its calibrated threshold, the true changepoint $\tau$, and the declared alarm $\hat\tau$.}
\label{fig:experiments}
\end{figure}

For a better resolution the online iterations evaluate $\widehat{\mathrm{MMD}}^2_n$
on a sliding window rather than the non-overlapping windows
the algorithm formally specifies, so consecutive $\tilde S_n$ are autocorrelated rather than i.i.d., and $\tau, \hat\tau$  are raw time indices rather than window counts. Threshold calibration
(Appendices~\ref{app:pilot-calibration}, \ref{app:calibration}) is carried
out under the same sliding construction, so the reported false-alarm
budget is internally consistent. The near-minimax delay-optimality proved
for non-overlapping windows is not guaranteed to transfer unchanged to the sliding construction.

All three delays land within a narrow band ($15$-$20$ steps) despite an
almost fourfold spread in $\hat\delta^2$, because the SR statistic reacts
exponentially fast once genuine evidence starts accumulating. 
The pilot signal strength mainly sets how many
windows it takes to {start} climbing, not how fast it climbs once it
does.

\section{A Real-Data Illustration: MNIST Digit 0 $\to$ 1}
\label{sec:real-data}

The three distribution pairs in Section~\ref{sec:experiments} are chosen so the nature of each shift is visually obvious. To check the
pipeline is not inadvertently relying on that transparency, we run it, unchanged, on an image domain. We use  MNIST digits~\citep{LeCunEtAl1998MNIST} down-sampled to $16\times16$ ($d=256$), $p_0$ as the digit ``0'', $p_1$ as the digit ``1'', with the same GRU history encoder now preceded by a small MLP frame embedding, and the same conditioned, pre-norm-residual, v-prediction denoiser architecture  (Appendix~\ref{app:architectures}).

The encoder-denoiser pair is trained with
classifier-free guidance~\citep{HoSalimans2022CFG} so the frozen context can also drive conditional generation, not only change-point detection. Window and
bandwidth follow the same convention $w=T_{\rm burn}$, $\sigma=\sqrt d$  as the synthetic pairs, giving $w=20$, $\sigma=16$. The threshold is calibrated in the same way (Approach B, $P_{\rm FA}\le5\%$ over a $105$-window horizon).

\begin{figure}[h]
\centering
\includegraphics[width=0.7\textwidth]{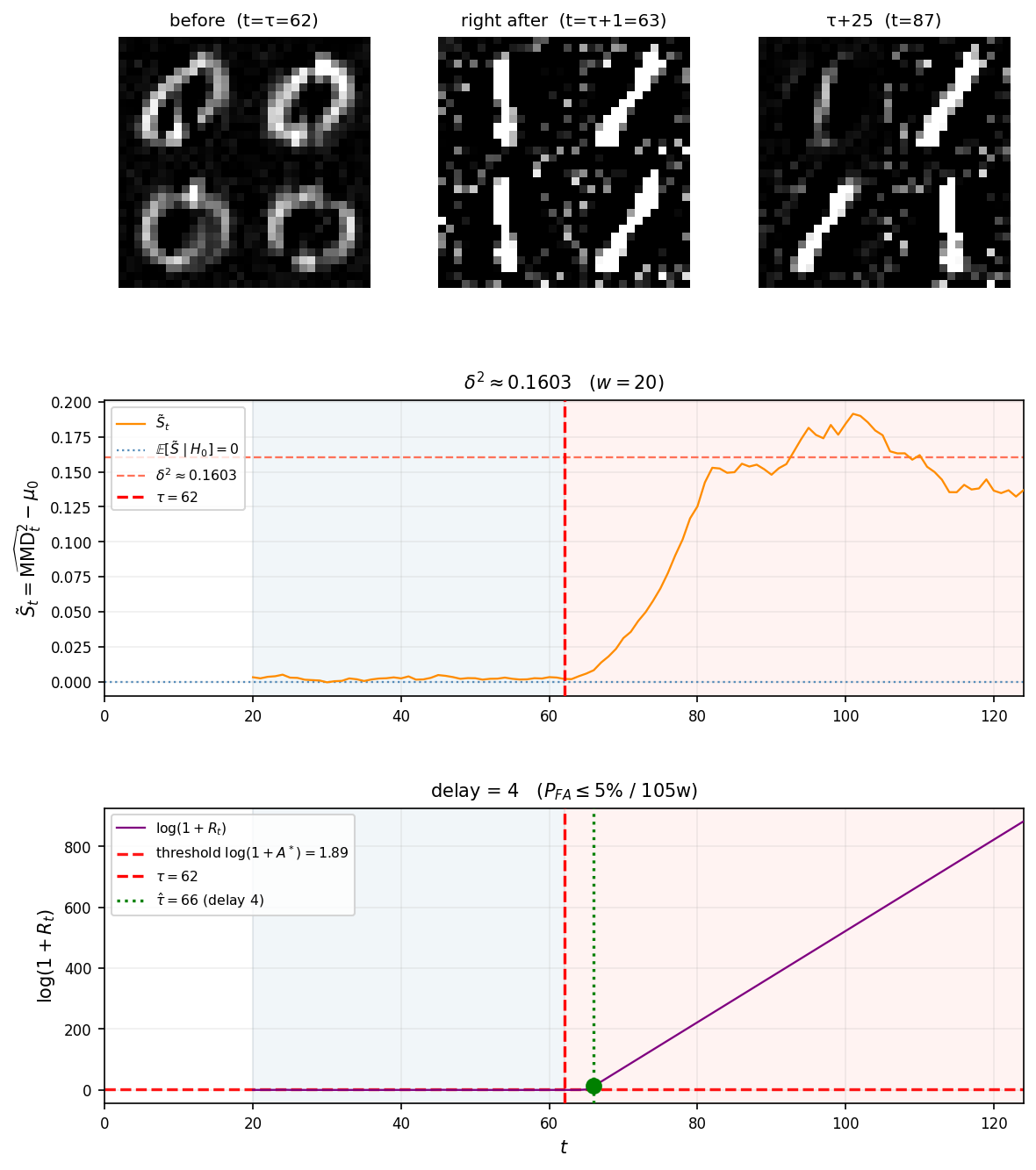}
\caption{MNIST digit $0\to1$, real data. Top: conditional generation
$p_\theta(x\mid\text{history}_{:t})$ just before $\tau$, one step after,
and $25$ steps after (four samples each, $2\times2$). Middle/bottom: the
same rolling MMD$^2$ and Shiryaev--Roberts panels as
Figure~\ref{fig:experiments}.}
\label{fig:mnist}
\end{figure}
 
Pilot calibration (described in Appendix~\ref{app:pilot-calibration}) on this reduced-budget model against a calibrated threshold
$\log(1+A^*)\approx1.89$ gives $\hat\delta^2\approx0.160$ which is slightly better compared to the strongest of the three synthetic pairs. 
Table~\ref{tab:experiments}).  
On a test series of length $125$ with a genuine change-point at $\tau=62$ 
the detector fires at $\hat\tau=66$ -- a delay of four windows, which is faster than any synthetic pair despite the much higher-dimensional observation space. This is consistent with detection speed tracking $\hat\delta^2$ rather than $d$ itself at fixed $w$. 

An interesting hypothesis is whether dimension works in the detector's favour in this example. Two clearly distinct digit classes,  pushed through the same encoder into a $256$-dimensional
latent space, end up far apart in the Euclidean distance 
whereas a hand-built 2D density shift has only two coordinates. 
The curse of dimensionality invoked in Section~\ref{sec:problem} affects the density estimator, however, it does not affect the signal that a well-trained encoder can carry. $\hat\delta^2$ does not shrink with $d$ the way a naive density estimation would. 

\section{Comparison with other methods}
\label{sec:baselines}


A natural question is how much of this performance comes from the PF-ODE encoder specifically, versus from the Shiryaev-Roberts machinery. We compare against three natural baselines on the pairs of Sections~\ref{sec:experiments}--\ref{sec:real-data} (the three synthetic pairs and MNIST).

\textbf{A}, the same two-sample $\widehat{\mathrm{MMD}}^2$ statistic computed directly on raw observations (no encoder), whose null law has no known closed form even asymptotically. Observations are standardised online by a rolling mean/std of its own history and its null approximated as $\mathcal N(0,1)$.

\textbf{B}, the mean-embedding statistic $\|\overline Z_w\|^2$ on the {same} PF-ODE latents, whose null is exactly $\tfrac1w\chi^2_d$ by the same injectivity argument as our full statistic, but which discards everything but the first moment;

\textbf{C}, classical Hotelling's $T^2$~\citep{Hotelling1931} (squared Mahalanobis distance of the window mean from a burn-in-estimated reference) fed into a one-sided CUSUM~\citep{Page1954}, whose null is asymptotically $\chi^2_d$ by the CLT. No PF-ODE, no MMD, no mixture-SR. 

Every method, including ours, is restricted to the same information. Only the burn-in segment of the series being monitored is ever treated as known-$p_0$ data; no baseline is given oracle access to the true generating $p_0$ (Appendix~\ref{app:baselines} gives the full construction).

Table~\ref{tab:baselines} reports the resulting metrics from $1{,}000$ independent Monte Carlo trials per dataset, each with a fresh realisation of the process and  change-point $\tau$ drawn uniformly at random after the burn-in window closes. Unlike Table~\ref{tab:experiments} in the Appendix, these numbers are not a single realized run but genuine estimates of false-alarm probability, miss probability, and mean detection delay, all at the shared $P_{\rm FA}\le5\%$ design budget. Only {our} method and method \textbf{B} carry a threshold that is {provably} calibrated to that budget, rather than approximately so.

\begin{table}[h]
\centering
\small
\begin{tabular}{llcccc}
\toprule
Experiment & Method & null dist. & $\mathbb{P}(\text{false alarm})$ & $\mathbb{P}(\text{miss})$ & avg. delay \\
\midrule
\multirow{4}{*}{Gaussian mixture rotation}
 & Ours & exact   & $\mathbf{1.9\%}$ & $9.7\%$  & $\mathbf{16.9}$ \\
 & A    & approx. & $32.0\%$          & $13.8\%$ & $30.9$ \\
 & B    & exact   & $24.1\%$          & $16.3\%$ & $36.2$ \\
 & C    & asympt. & $67.0\%$          & $\mathbf{2.8}\%$  & $38.5$ \\
\midrule
\multirow{4}{*}{Four clusters $\to$ one}
 & Ours & exact   & $\mathbf{2.9\%}$ & $15.3\%$  & $\mathbf{29.3}$ \\
 & A    & approx. & $44.3\%$          & $14.5\%$ & $40.6$ \\
 & B    & exact   & $0.0\%$           & $100\%$\textsuperscript{$\ddagger$}  & --- \\
 & C    & asympt. & $79.7\%$          & $\mathbf{0.7}\%$  & $31.8$ \\
\midrule
\multirow{4}{*}{Blob $\to$ Ring}
 & Ours & exact   & $\mathbf{1.8\%}$ & $9.2\%$   & $15.3$ \\
 & A    & approx. & $31.7\%$          & $5.0\%$  & $\mathbf{14.2}$ \\
 & B    & exact   & $23.2\%$          & $13.1\%$ & $31.8$ \\
 & C    & asympt. & $57.0\%$          & $\mathbf{3.1}\%$  & $26.7$ \\
\midrule
\multirow{5}{*}{MNIST $0\to1$}
 & Ours (DDIM inversion)          & exact   & $18.7\%$          & $4.2\%$  & $\mathbf{6.7}$ \\
 & Ours (EDICT inversion)\textsuperscript{$\dagger$} & exact   & $\mathbf{16.5\%}$ & $4.5\%$  & $6.9$ \\
 & A    & approx. & $18.5\%$          & $5.1\%$  & $8.1$ \\
 & B    & exact   & $50.9\%$          & $\mathbf{2.3}\%$  & $9.1$ \\
 & C    & asympt. & $98.7\%$          & $0.0\%$\textsuperscript{$\P$}  & $0.0$\textsuperscript{$\P$} \\
\bottomrule
\end{tabular}
\caption{Operating characteristics over $1{,}000$ Monte Carlo trials per
experiment ($\tau$ uniform after burn-in, shared $P_{\rm FA}\le5\%$
design budget). $\dagger$ Same encoder-denoiser checkpoint as the row above, only the DDIM encoding step at detection time replaced by
EDICT~\citep{WallaceEtAl2023EDICT} -- no retraining involved.}
\label{tab:baselines}
\end{table}

On the three synthetic pairs, the picture is unambiguous: Ours
is the only method that stays near the $5\%$ design budget (observed $1.8$--$2.9\%$), while every baseline overshoots it by roughly an order of magnitude ($24$--$80\%$). \textbf{A}'s $\mathcal N(0,1)$ approximation to its own null is not accurate enough in finite windows and over-alarms substantially. \textbf{B} is either mediocre 
or outright non-operational  when the shift is mean-preserving (the case of ``Four clusters $\to$ one'', where $\hat\delta^2<0$, so it never fires and gives a $100\%$ miss rate).
\textbf{C}'s asymptotic $\chi^2_d$ null is the most severely miscalibrated of all, since Hotelling's $T^2$ needs an approximately Gaussian window mean that a bimodal $p_0$ at $w=25$ does not deliver.

MNIST is the only exception. No method holds the $5\%$ budget. {Ours} overshoots it substantially.  
The EDICT inversion of \cite{WallaceEtAl2023EDICT} allows for a slight improvement but does not fix the problem. 
The underlying issue is that DDIM inversion of the burn-in observations does not recover a clean $\mathcal N(0,I)$ latent. The mean per-dimension variance of the recovered code $Z$ is $0.845$, not $1$.

This symptom is visible in Figure~\ref{fig:mnist} which shows the conditional samples generated from the frozen context with speckled noise around and behind the digit strokes instead of a clean digit. This is the visual representation of a score network whose denoising has not fully converged. 
EDICT did not help much because 
it deliberately trades exact invertibility for numerical stability. 

\section{Conclusion and Future Directions}
\label{sec:conclusion}
We address sequential change-point detection problem when neither the pre- nor post-change conditional density admits a closed form. The construction stacks three components. First, a conditional PF-ODE diffusion model, trained only on pre-change data with a frozen history context, gives a provably injective encoder, which maps the unknown data distribution onto $\mathcal N(0,I)$. Second, MMD$^2$ against this reference reduces detection to a single scalar per window, which is degenerate under $H_0$ but $\sqrt w$-Gaussian under $H_1$. 
Finally, the Tartakovsky-Spivak mixture of the Shiryaev-Roberts statistic, with Pollak-calibration to a stated false-alarm budget 
turns that sequence into a sequential decision rule.

Validated on three synthetic 
shifts the method detected all three within a narrow $15$-$20$-window delay despite an almost fourfold spread in signal strength $\hat\delta^2$. 
When run unchanged on real MNIST data ($d=256$), it detected a $0\to1$ digit-class
change in just four windows, which confirms
that dimension enters only the encoding cost, not the behaviour 
of the resulting scalar statistic $\tilde S_t$. So, a well-separated high-dimensional shift can be easier, not harder, to detect than a low-dimensional one. 

The method needs only 
a trainable conditional encoder, so it applies wherever sequential data
has 
an unknown high-dimensional conditional distribution. This includes financial return or order-books, where heavy tails and high-dimensionality make closed-form modelling impractical, auto-segmentation of video and audio into scenes, speakers, or musical sections, 
detecting abrupt
shifts in motion-sensor and self-driving telemetry, and so on. 

{The denoiser needs to be trained to a genuinely clean null.} That is, 
$Z\sim\mathcal N(0,I)$ under $H_0$  in Lemma~\ref{lem:pushforward-injective} has to come from a conditional, autoregressively-trained denoiser that has actually converged. The MNIST example in Section~\ref{sec:baselines} shows that this is the binding constraint in practice. 
This is the most consequential open problem, since every other component 
is conditioned on the encoder actually reaching that asymptotic state. 


The experiments commit to one fixed-bandwidth RBF kernel.  Table~\ref{tab:kernels} shows that the kernel choice determines which shifts are detectable at all, so adaptive or learned kernels 
could sharpen sensitivity. 
Finally, Algorithm~\ref{alg:sr-detection} targets a single $\theta$. Extending it to continuous monitoring by restarting the context and SR statistic at each alarm, as sketched in Appendix~\ref{app:algorithm}, 
and proving formal guarantees for the restart procedure remain open.

\newpage
\bibliographystyle{apalike}
\bibliography{literature}

\newpage
\appendix

\section{Injectivity of the Conditional PF-ODE Encoder}
\label{app:bijection}

\begin{lemma}
\label{lem:pushforward-injective}
Let $\Phi(\cdot;\mathbf{h}) : \mathbb{R}^d \to \mathbb{R}^d$ be injective and Borel measurable, with Borel measurable inverse on its range. Then for any two probability measures $P, P'$ on $\mathbb{R}^d$,
$$\Phi(\cdot;\mathbf{h})_\# P = \Phi(\cdot;\mathbf{h})_\# P' \quad\Longleftrightarrow\quad P = P'.$$
\end{lemma}

\begin{proof}
The ($\Leftarrow$) direction is immediate. For ($\Rightarrow$), apply $\bigl(\Phi(\cdot;\mathbf{h})^{-1}\bigr)_\#$ to both sides:
$$\bigl(\Phi(\cdot;\mathbf{h})^{-1}\bigr)_\#\Phi(\cdot;\mathbf{h})_\# P = \bigl(\Phi(\cdot;\mathbf{h})^{-1}\circ\Phi(\cdot;\mathbf{h})\bigr)_\# P = P,$$
and likewise for $P'$, since $\Phi(\cdot;\mathbf{h})^{-1}\circ\Phi(\cdot;\mathbf{h}) = \mathrm{id}$ on the domain. Hence $P = P'$.
\end{proof}

Since $\Phi(\cdot;\mathbf{h}_{\rm fixed})_\# p_0 = \mathcal{N}(0,I)$ by construction (Section~\ref{sec:problem}) and $p_1 \neq p_0$, Lemma~\ref{lem:pushforward-injective} gives $\Phi(\cdot;\mathbf{h}_{\rm fixed})_\# p_1 \neq \mathcal{N}(0,I)$: no genuine distributional shift can be rendered invisible by the encoding, regardless of which moments or features it affects.

\begin{remark}
This guarantee requires only \emph{injectivity} of $\Phi(\cdot;\mathbf{h})$, not full surjectivity. PF-ODE flows obtain injectivity for free from the uniqueness of ODE trajectories (Picard--Lindelöf): distinct initial conditions cannot collide, so $\Phi(\cdot;\mathbf{h})$ is automatically invertible whenever the drift $f(x,\tau) - \tfrac{g^2(\tau)}{2}\nabla_x\log p_\tau(x)$ is Lipschitz in $x$. This is not architectural — no invertible-network design (as in normalizing flows such as RealNVP~\citep{DinhEtAl2017RealNVP} or Glow~\citep{KingmaDhariwal2018Glow}) is required.

In contrast, the stochastic reverse process of DDPM~\citep{HoEtAl2020DDPM} injects independent noise $\eta_\tau \sim \mathcal{N}(0,I)$ at every step, making the map a Markov kernel rather than a function: the same $X_t$ can be encoded into different latent codes across runs, and no measurable inverse exists. Lemma~\ref{lem:pushforward-injective} therefore does not apply, and a DDPM-based encoding offers no guarantee that a genuine change $p_1 \neq p_0$ will produce $P_1 \neq \mathcal{N}(0,I)$ --- the shift could, in principle, be averaged away by the injected noise.
\end{remark}

\section{Derivation of Terms $B$ and $C$ for the RBF Kernel}
\label{app:BC-derivation}

\subsection*{Term C}

$$C = \mathbb{E}_{Y,Y' \sim \mathcal{N}(0,I)}\!\left[
  \exp\!\left(-\frac{\|Y - Y'\|^2}{2\sigma^2}\right)\right]$$

Since $Y \perp Y'$ and each has independent coordinates,
$Y - Y' \sim \mathcal{N}(0, 2I)$. The kernel factorises over dimensions:
$$C = \prod_{l=1}^{d}
  \mathbb{E}_{\xi_l \sim \mathcal{N}(0,2)}\!\left[
  \exp\!\left(-\frac{\xi_l^2}{2\sigma^2}\right)\right]$$

For a single factor, we apply the moment generating function identity
for $\xi \sim \mathcal{N}(0, \sigma_\xi^2)$:
$$\mathbb{E}\!\left[e^{-t\xi^2}\right] = \bigl(1 + 2t\sigma_\xi^2\bigr)^{-1/2}$$

Setting $t = 1/(2\sigma^2)$ and $\sigma_\xi^2 = 2$:
$$\mathbb{E}_{\xi \sim \mathcal{N}(0,2)}\!\left[\exp\!\left(-\frac{\xi^2}{2\sigma^2}\right)\right]
= \left(1 + \frac{2}{\sigma^2}\right)^{-1/2}
= \frac{\sigma}{\sqrt{\sigma^2+2}}$$

Taking the product over $d$ independent dimensions:
$${C = \left(\frac{\sigma^2}{\sigma^2 + 2}\right)^{d/2}}$$

\subsection*{Term B}

For a fixed observation $z \in \mathbb{R}^d$:
$$B(z) = \mathbb{E}_{Y \sim \mathcal{N}(0,I)}\!\left[
  \exp\!\left(-\frac{\|z - Y\|^2}{2\sigma^2}\right)\right]$$

The kernel again factorises over coordinates:
$$B(z) = \prod_{l=1}^{d}\,
  \mathbb{E}_{Y_l \sim \mathcal{N}(0,1)}\!\left[
  \exp\!\left(-\frac{(z_l - Y_l)^2}{2\sigma^2}\right)\right]$$

For each factor, we compute the integral by combining the two Gaussian terms.
Expanding in the exponent:
$$-\frac{(z_l - y)^2}{2\sigma^2} - \frac{y^2}{2}
= -\frac{1}{2\sigma^2}\bigl(z_l^2 - 2z_l y + y^2\bigr) - \frac{y^2}{2}
= -\frac{\sigma^2 + 1}{2\sigma^2}\,y^2 + \frac{z_l}{\sigma^2}\,y - \frac{z_l^2}{2\sigma^2}$$

Completing the square in $y$:
$$= -\frac{\sigma^2+1}{2\sigma^2}\!\left(y - \frac{z_l}{\sigma^2+1}\right)^2
  + \frac{z_l^2}{2(\sigma^2+1)} - \frac{z_l^2}{2\sigma^2}
= -\frac{\sigma^2+1}{2\sigma^2}\!\left(y - \frac{z_l}{\sigma^2+1}\right)^2
  - \frac{z_l^2}{2(\sigma^2+1)}$$

The one-dimensional integral therefore evaluates to:
\begin{align*}
&\frac{1}{\sqrt{2\pi}}\int_{-\infty}^{\infty}
  \exp\!\left(-\frac{(z_l-y)^2}{2\sigma^2} - \frac{y^2}{2}\right)dy \\
&= \exp\!\left(-\frac{z_l^2}{2(\sigma^2+1)}\right)
  \cdot \frac{1}{\sqrt{2\pi}}\int_{-\infty}^{\infty}
  \exp\!\left(-\frac{\sigma^2+1}{2\sigma^2}
  \left(y - \frac{z_l}{\sigma^2+1}\right)^2\right)dy \\
&= \exp\!\left(-\frac{z_l^2}{2(\sigma^2+1)}\right)
  \cdot \sqrt{\frac{\sigma^2}{\sigma^2+1}}
\end{align*}

where the last Gaussian integral gives $\sqrt{2\pi\sigma^2/(\sigma^2+1)}$.
Taking the product over all $d$ dimensions:
$${B(z) = \left(\frac{\sigma^2}{\sigma^2 + 1}\right)^{d/2}
  \exp\!\left(-\frac{\|z\|^2}{2(\sigma^2 + 1)}\right)}$$

\subsection*{Bias of the diagonal-included estimator}
\label{app:bias-derivation}

Under $H_0$ ($Z_i \overset{\text{iid}}{\sim} P_0 = Q$), split Term $A$ into its diagonal and off-diagonal parts:
$$\frac{1}{w^2}\sum_{i,j=1}^{w} k(Z_i,Z_j)
= \frac{1}{w^2}\sum_{i\neq j} k(Z_i,Z_j) + \frac{1}{w^2}\sum_{i=1}^{w} k(Z_i,Z_i).$$
For $i\neq j$, $Z_i \perp Z_j$ and both are drawn from $Q$, so $\mathbb{E}_0[k(Z_i,Z_j)] = \mathbb{E}_{Y,Y'\sim Q}[k(Y,Y')] = C$; there are $w(w-1)$ such ordered pairs, contributing $\tfrac{w-1}{w}C$. On the diagonal, $k(Z_i,Z_i) = k_0 := k(z,z)$, a constant for any stationary kernel (for the RBF kernel, $k_0 = 1$ regardless of $z$); the $w$ diagonal terms contribute $k_0/w$. For Term $B$, since $Z_i \sim Q$ under $H_0$,
$$\mathbb{E}_0[B(Z_i)] = \mathbb{E}_{Z\sim Q}\bigl[\mathbb{E}_{Y\sim Q}[k(Z,Y)]\bigr] = \mathbb{E}_{Z,Y\sim Q}[k(Z,Y)] = C,$$
so the $w$ terms of Term $B$ contribute $-2C$. Term $C$ is already the constant $C$ itself. Summing all three:
$$\mathbb{E}_0\bigl[\widehat{\mathrm{MMD}}^2\bigr]
= \underbrace{\frac{w-1}{w}C + \frac{k_0}{w}}_{\text{Term }A}
\;\underbrace{-\,2C}_{\text{Term }B}
\;+\;\underbrace{C}_{\text{Term }C}
= \frac{k_0-C}{w},$$
an $O(1/w)$ bias coming entirely from the $w$ diagonal terms $k(Z_i,Z_i)$ retained in Term $A$; removing them, as in $\widehat{\mathrm{MMD}}^2_u$ below, removes the bias exactly rather than merely asymptotically. The analogous bias of the diagonal-included (``V-statistic'') estimator in the two-sample setting is well known; see e.g.\ \citet{GrettonEtAl2012}.

\section{Asymptotic Distribution of $\widehat{\mathrm{MMD}}^2$: Derivations}
\label{app:mmd-asymptotics}

\subsection*{Step 1 --- Centered kernel and degeneracy}

Define the \emph{centered kernel}:
$$\tilde{k}(x,y) = k(x,y) - B(x) - B(y) + C$$

where $B(z) = \mathbb{E}_{Y \sim Q}[k(z,Y)]$ is Term B evaluated at $z$.
The unbiased estimator is then a degree-2 U-statistic with kernel $\tilde{k}$:
$$\widehat{\mathrm{MMD}}^2_u = \frac{1}{w(w-1)}\sum_{i \neq j}\tilde{k}(Z_i, Z_j)$$

Under $H_0$ ($Z \sim P_0 = Q$), the centered kernel satisfies the
\emph{first-order degeneracy condition}:
$$\mathbb{E}_{Z' \sim P_0}\bigl[\tilde{k}(z, Z')\bigr]
= \mathbb{E}_{Z'}[k(z,Z')] - B(z) - \mathbb{E}_{Z'}[B(Z')] + C
= B(z) - B(z) - C + C = 0 \quad \forall\, z$$

Since this holds for \emph{every} fixed $z$, it holds in particular after
averaging over $z = Z_i \sim P_0$ itself: by the tower property, for each
pair $i \neq j$,
$$\mathbb{E}_0\bigl[\tilde{k}(Z_i,Z_j)\bigr]
= \mathbb{E}_0\Bigl[\underbrace{\mathbb{E}_{Z_j}\bigl[\tilde{k}(Z_i,Z_j)\mid Z_i\bigr]}_{=\,0\text{ by the above, at }z=Z_i}\Bigr] = 0,$$
so the unbiased estimator has \emph{exactly} zero mean under $H_0$ for every $w$:
$$\mathbb{E}_0\bigl[\widehat{\mathrm{MMD}}^2_u\bigr]
= \frac{1}{w(w-1)}\sum_{i\neq j}\mathbb{E}_0\bigl[\tilde{k}(Z_i,Z_j)\bigr] = 0.$$

For an ordinary (non-degenerate) U-statistic the standard CLT gives
a $\sqrt{w}$-rate and a Gaussian limit. For a first-order degenerate U-statistic
the leading Gaussian term vanishes identically, and the fluctuation is of order $1/w$.
Correct normalisation requires scaling by $w$:
$$w\cdot\widehat{\mathrm{MMD}}^2_u = \frac{w}{w(w-1)}\sum_{i \neq j}\tilde{k}(Z_i,Z_j)$$

\subsection*{Step 2 --- Spectral decomposition and limiting distribution}

The centered kernel $\tilde{k}$ is symmetric and square-integrable under $P_0$.
Define the integral operator $T_{\tilde{k}} : L^2(P_0) \to L^2(P_0)$:
$$(T_{\tilde{k}}\,\phi)(x) = \int \tilde{k}(x,y)\,\phi(y)\,dP_0(y) = \lambda\,\phi(x)$$

By the Hilbert--Schmidt theorem, $T_{\tilde{k}}$ is compact and self-adjoint,
so it has a countable spectral decomposition. The degeneracy condition forces
all eigenfunctions $\phi_l$ to satisfy $\mathbb{E}_{P_0}[\phi_l(Z)] = 0$ (zero mean),
with $\mathrm{Var}[\phi_l(Z)] = 1$ by normalisation. The Mercer expansion gives:
$$\tilde{k}(x,y) = \sum_{l=1}^{\infty} \lambda_l\,\phi_l(x)\,\phi_l(y),
\qquad \lambda_l \geq 0, \quad \sum_{l}\lambda_l^2 < \infty$$

Substituting into the rescaled U-statistic and using the symmetry of $\tilde{k}$:
$$w\cdot\widehat{\mathrm{MMD}}^2_u
\approx \frac{1}{w}\sum_{i \neq j}\tilde{k}(Z_i,Z_j)
= \sum_{l=1}^{\infty} \lambda_l
\left[\underbrace{\left(\frac{1}{\sqrt{w}}\sum_{i=1}^{w}\phi_l(Z_i)\right)^2}_{=:\,S_l^2}
- \underbrace{\frac{1}{w}\sum_{i=1}^{w}\phi_l(Z_i)^2}_{\to\,1\text{ by LLN}}\right]$$

By the CLT, $S_l \xrightarrow{d} \mathcal{N}(0,1)$ for each $l$ (since $\mathbb{E}[\phi_l] = 0$,
$\mathrm{Var}[\phi_l] = 1$). The orthogonality $\mathbb{E}[\phi_l(Z)\phi_m(Z)] = \delta_{lm}$
implies $S_l \perp S_m$ asymptotically for $l \neq m$. Together with Slutsky's theorem:
$${w\cdot\widehat{\mathrm{MMD}}^2_u \;\xrightarrow{d}\;
\sum_{l=1}^{\infty} \lambda_l\,(Z_l^2 - 1),
\qquad Z_l \overset{\mathrm{iid}}{\sim} \mathcal{N}(0,1)}$$

This result was established for degenerate U-statistics by \citet{Hall1984};
the MMD-specific statement appears in \citet[Theorem~12]{GrettonEtAl2012}.
The limit belongs to the class of \emph{generalised $\chi^2$ distributions}
(infinite weighted sums of centred $\chi^2(1)$ variables), for which no
closed-form CDF exists.

\subsection*{Step 3 --- Eigenvalues for the RBF kernel under $P_0 = \mathcal{N}(0,I)$}

Consider first $d=1$. We seek $\lambda,\phi$ solving the eigenvalue problem for the (uncentred) kernel $k(x,y)=e^{-(x-y)^2/2\sigma^2}$ against $P_0=\mathcal N(0,1)$:
$$\int_{-\infty}^{\infty} k(x,y)\,\phi(y)\,\frac{e^{-y^2/2}}{\sqrt{2\pi}}\,dy=\lambda\,\phi(x).$$
Trying the ground-state ansatz $\phi_0(x)=e^{-\beta x^2}$ and evaluating the Gaussian integral (complete the square in $y$, then apply $\int e^{-Ky^2+Ly}dy=\sqrt{\pi/K}\,e^{L^2/4K}$) gives, after matching the coefficient of $x^2$ on both sides, a self-consistency condition for $K:=\tfrac{1}{2\sigma^2}+\beta+\tfrac12$:
$$K^2-\left(\frac{1}{\sigma^2}+\frac12\right)K+\frac{1}{4\sigma^4}=0.$$
Taking the admissible (larger) root,
$$K = A := p+q+u, \qquad p=\frac14,\quad q=\frac{1}{2\sigma^2},\quad u=\sqrt{p^2+2pq}=\frac{\sqrt{\sigma^2+4}}{4\sigma},$$
which matches the classical parametrisation of Gaussian-kernel/Gaussian-measure eigenproblems (\citealt{ZhuEtAl1997}; \citealt[\S4.3.1]{RasmussenWilliams2006}). The ladder structure of the Hermite/harmonic-oscillator basis propagates this ratio to every mode: writing the (unnormalised) eigenfunctions as weighted Hermite functions,
$$\phi_n(x)\propto He_n\!\bigl(c(\sigma)\,x\bigr)\,e^{-a(\sigma)x^2}, \qquad a(\sigma)=u-p,\quad c(\sigma)=2\sqrt{u},$$
the eigenvalues decay \emph{exactly geometrically}, $\lambda_n=\lambda_0\,r^{\,n}$, with ratio
$$r=\frac{q}{A}=\frac{1/(2\sigma^2)}{1/4+1/(2\sigma^2)+\sqrt{\sigma^2+4}/(4\sigma)}=\frac{2}{\sigma^2+2+\sigma\sqrt{\sigma^2+4}}.$$
Rationalising (multiply numerator and denominator by $\sigma^2+2-\sigma\sqrt{\sigma^2+4}$, whose product with the denominator collapses to $4$) gives the compact closed form
$${r=\frac{\sigma^2+2-\sigma\sqrt{\sigma^2+4}}{2}\in(0,1)}.$$
As a check, $r\to1$ as $\sigma\to0$ (a near-delta kernel needs infinitely many modes to resolve) and $r\to0$ as $\sigma\to\infty$ (a near-constant kernel degenerates to rank one) --- both limits behave as they should, which rules out the naive guess $r=\sigma^2/(\sigma^2+2)$ (that expression is in fact $C^{2/d}$ from Step 1, an unrelated quantity, and has the opposite monotonicity in $\sigma$).

For general $d$, the kernel factorises over coordinates, so the eigenfunctions are products of one-dimensional weighted Hermite functions indexed by a multi-index $\mathbf n=(n_1,\ldots,n_d)$,
$$\phi_\mathbf{n}(x) \propto \prod_{l=1}^{d}
He_{n_l}\!\bigl(c(\sigma)\,x_l\bigr)\cdot e^{-a(\sigma)x_l^2},
\qquad \lambda_\mathbf{n} \propto r^{|\mathbf{n}|},$$
and all $\binom{m+d-1}{d-1}$ eigenfunctions of total degree $m=|\mathbf n|$ share the same eigenvalue:
$$\lambda_m \propto r^m = \left(\frac{\sigma^2+2-\sigma\sqrt{\sigma^2+4}}{2}\right)^{m}.$$

The series $\sum_m \binom{m+d-1}{d-1}\lambda_m^2 < \infty$ ensures Hilbert--Schmidt
compactness, and the geometric decay means the limiting distribution is
dominated by the first few eigenfunctions.

\begin{table}[h]
\centering
\small
\setlength{\tabcolsep}{4pt}
\renewcommand{\arraystretch}{1.6}
\begin{tabular}{lp{2.3cm}p{2.4cm}p{2.4cm}p{3.0cm}p{2.6cm}}
\toprule
\textbf{Kernel} &
$k(x,y)$ &
$C$ &
$B$ &
$w\cdot\widehat{\mathrm{MMD}}^2\xrightarrow{d}$ &
\textbf{Detectable shifts} \\
\midrule

Linear &
$x^\top y$ &
$0$ &
$0$ &
$\chi^2(d)$\newline(exact, any $w$) &
Mean only:\newline$\mathbb{E}[Z] \neq 0$ \\[4pt]

Deg.-$p$ poly. &
$(1 + x^\top y/c)^p$ &
$\displaystyle\sum_{m=0}^{\lfloor p/2\rfloor}\kappa_m$ &
$\displaystyle\sum_{m=0}^{\lfloor p/2\rfloor}\tau_m(z)$ &
$\sum_{l=1}^{K}\lambda_l(Z_l^2-1)$,\newline
$K = \binom{p+d}{d}-1$ &
Moments $1,\ldots,p$:\newline
mean, cov.,\newline
skewness $(p\!\geq\!3)$ \\[4pt]

\textbf{RBF} &
$e^{-\|x-y\|^2/2\sigma^2}$ &
$\left(\dfrac{\sigma^2}{\sigma^2+2}\right)^{d/2}$ &
$\left(\dfrac{\sigma^2}{\sigma^2+1}\right)^{d/2}$\newline
$\cdot\,e^{-\|z\|^2/2(\sigma^2+1)}$ &
$\sum_{l=1}^{\infty}\lambda_l(Z_l^2-1)$,\newline
$\lambda_l \propto r^l$ &
Any $P\!\neq\!Q$  \\[4pt]

Laplace &
$e^{-\|x-y\|/\sigma}$ &
No closed\newline form &
No closed\newline form &
$\sum_{l=1}^{\infty}\lambda_l(Z_l^2-1)$,\newline
$\lambda_l \sim l^{-(d+1)/2}$ &
Any $P\!\neq\!Q$ \\[4pt]

IMQ &
$(c^2\!+\!\|x\!-\!y\|^2)^{-\beta}$ &
No closed\newline form &
No closed\newline form &
$\sum_{l=1}^{\infty}\lambda_l(Z_l^2-1)$,\newline
exp.\ decay &
Any $P\!\neq\!Q$ \\

\bottomrule
\end{tabular}
\caption{$\mathrm{MMD}^2$ against $Q = \mathcal{N}(0,I)$ for common kernels.}
\label{tab:kernels}
\end{table}

\subsection*{Alternative case --- Hoeffding decomposition}
\label{app:mmd-alt}

For a symmetric kernel $h(x,y)$ and an i.i.d.\ sample $Z_1,\ldots,Z_w$ drawn from
a common law -- $P_0$ under $H_0$, $P_1$ under $H_1$ -- the U-statistic
$U_w = \frac{1}{w(w-1)}\sum_{i\neq j}h(Z_i,Z_j)$ admits
the orthogonal \citet{Hoeffding1948} decomposition:
$$h(z_1, z_2) = \mu + h_1(z_1) + h_1(z_2) + h_2(z_1,z_2)$$
where, with expectations taken under that same sampling law,
\begin{align*}
\mu &= \mathbb{E}[h(Z,Z')], \\
h_1(z) &= \mathbb{E}_{Z'}[h(z,Z')] - \mu, \\
h_2(z_1,z_2) &= h(z_1,z_2) - h_1(z_1) - h_1(z_2) - \mu.
\end{align*}

By construction, $\mathbb{E}[h_1(Z)] = 0$ and $\mathbb{E}_{Z'}[h_2(z,Z')] = 0$
for every fixed $z$, and the three components are mutually orthogonal.
Substituting into the centered statistic:
$$U_w - \mu
= \underbrace{\frac{2}{w}\sum_{i=1}^{w} h_1(Z_i)}_{\text{linear term}}
+ \underbrace{\frac{1}{w(w-1)}\sum_{i\neq j} h_2(Z_i,Z_j)}_{\text{degenerate U-statistic}}$$

In our setting $h = \tilde{k}$, the kernel centred with respect to $Q = \mathcal{N}(0,I)$,
so $\mu = \mathrm{MMD}^2(P_0,Q)$ under $H_0$ and $\mu = \mathrm{MMD}^2(P_1,Q)$ under $H_1$.
Under the null $P_0 = Q$, the first-order projection evaluates to:
$$h_1(z) = \mathbb{E}_{Z'\sim Q}[\tilde{k}(z,Z')] - 0
= \underbrace{\mathbb{E}_{Z'\sim Q}[k(z,Z')]}_{=\,B(z)} - B(z)
  - \underbrace{\mathbb{E}_{Z'\sim Q}[B(Z')]}_{=\,C} + C = 0$$

Hence $h_1 \equiv 0$ on all of $\mathbb{R}^d$: this is precisely \emph{first-order degeneracy}.
The linear term vanishes identically, and the statistic is governed entirely
by the second-order degenerate part, which after rescaling by $w$ converges
to the generalised $\chi^2$ distribution described above.

Suppose now $Z_1,\ldots,Z_w$ are instead drawn i.i.d.\ from $P_1 \neq Q$, with
signal strength $\delta^2 = \mathrm{MMD}^2(P_1, Q) > 0$ taking the role of $\mu$.
The first-order projection becomes:
$$h_1(z) = \mathbb{E}_{Z'\sim P_1}[\tilde{k}(z,Z')] - \delta^2
= \underbrace{\mathbb{E}_{Z'\sim P_1}[k(z,Z')]}_{\neq\,B(z)}
  - B(z) - \mathbb{E}_{Z'\sim P_1}[B(Z')] + C - \delta^2 \;\not\equiv\; 0$$

since $\mathbb{E}_{Z'\sim P_1}[k(z,Z')] \neq \mathbb{E}_{Z'\sim Q}[k(z,Z')] = B(z)$
whenever $P_1 \neq Q$.
Let $\sigma_1^2 = \mathrm{Var}_{Z\sim P_1}[h_1(Z)] > 0$.

The linear term is now a sum of $w$ independent, mean-zero contributions:
$$\frac{2}{w}\sum_{i=1}^{w} h_1(Z_i)
= \frac{2}{\sqrt{w}} \cdot \frac{1}{\sqrt{w}}\sum_{i=1}^{w} h_1(Z_i)
\;\xrightarrow{d}\; \mathcal{N}\!\left(0,\,\frac{4\sigma_1^2}{w}\right)$$

by the standard CLT, since $\mathbb{E}[h_1(Z_i)] = 0$ and
$\mathrm{Var}[h_1(Z_i)] = \sigma_1^2 < \infty$.
This contributes at rate $O_p(w^{-1/2})$.

The second-order term $\frac{1}{w(w-1)}\sum_{i\neq j}h_2(Z_i,Z_j)$ remains
a degenerate U-statistic (since $\mathbb{E}_{Z'}[h_2(z,Z')] = 0$ by construction)
and satisfies $\mathrm{Var}\bigl[\frac{1}{w(w-1)}\sum_{i\neq j}h_2\bigr] = O(w^{-2})$,
placing it at rate $O_p(w^{-1}) \ll O_p(w^{-1/2})$.
By Slutsky's theorem the degenerate term is asymptotically negligible, giving:

$${\sqrt{w}\!\left(\widehat{\mathrm{MMD}}^2_u - \delta^2\right)
\;\xrightarrow{d}\; \mathcal{N}\!\left(0,\; 4\sigma_1^2\right),
\qquad \sigma_1^2 = \mathrm{Var}_{Z\sim P_1}\!\left[\mathbb{E}_{Z'\sim P_1}[\tilde{k}(Z,Z')]\right]}$$

\section{Tartakovsky--Spivak Mixture Likelihood Ratio under the Exponential Prior}
\label{app:tartakovsky-spivak}

Under $H_1$ conditional on signal strength $\delta^2$, the window statistic
is approximately $\tilde S_t \sim \mathcal N(\delta^2, v_1^2)$. Averaging this
Gaussian density over the exponential prior $\pi(\delta^2) = \alpha
e^{-\alpha\delta^2}$, $\delta^2>0$, gives the marginal (mixture) density
$$\bar{p}_1(s) = \int_0^{\infty}
\frac{1}{\sqrt{2\pi}\,v_1}\exp\!\left(-\frac{(s-\delta^2)^2}{2v_1^2}\right)
\alpha e^{-\alpha\delta^2}\,d(\delta^2).$$

Write $x = \delta^2$ for brevity and expand the exponent:
$$-\frac{(s-x)^2}{2v_1^2} - \alpha x
= -\frac{1}{2v_1^2}\Bigl[x^2 - 2x(s - \alpha v_1^2) \Bigr] - \frac{s^2}{2v_1^2}.$$

Completing the square in $x$:
$$= -\frac{\bigl(x-(s-\alpha v_1^2)\bigr)^2}{2v_1^2}
+ \frac{(s-\alpha v_1^2)^2 - s^2}{2v_1^2}.$$

The quadratic term in the exponent simplifies to a term linear in $s$:
$$\frac{(s-\alpha v_1^2)^2 - s^2}{2v_1^2}
= \frac{-2\alpha v_1^2 s + \alpha^2 v_1^4}{2v_1^2}
= -\alpha s + \frac{\alpha^2 v_1^2}{2}.$$

Substituting back,
$$\bar p_1(s) = \alpha \exp\!\left(-\alpha s + \frac{\alpha^2 v_1^2}{2}\right)
\underbrace{\int_0^\infty \frac{1}{\sqrt{2\pi}\,v_1}
\exp\!\left(-\frac{\bigl(x-(s-\alpha v_1^2)\bigr)^2}{2v_1^2}\right) dx}_{=\;\mathbb P\bigl(\mathcal N(s-\alpha v_1^2,\,v_1^2) > 0\bigr)}.$$

The remaining integral is the probability that a $\mathcal N(s-\alpha v_1^2,
v_1^2)$ variable exceeds $0$ -- the boundary of the prior's support -- i.e.\
$$\int_0^\infty \frac{1}{\sqrt{2\pi}\,v_1}
\exp\!\left(-\frac{\bigl(x-(s-\alpha v_1^2)\bigr)^2}{2v_1^2}\right) dx
= \Phi\!\left(\frac{s-\alpha v_1^2}{v_1}\right).$$

Collecting terms gives the closed-form mixture density
$${\bar{p}_1(s)
= \alpha\,\exp\!\left(-\alpha s + \frac{\alpha^2 v_1^2}{2}\right)
\Phi\!\left(\frac{s - \alpha v_1^2}{v_1}\right)}.$$

Dividing by the null density $g_0$ of $G_0$ -- evaluated numerically, since
$G_0$ has no closed form (Appendix~\ref{app:mmd-asymptotics}) -- gives the
mixture likelihood ratio used in the Shiryaev--Roberts recursion:
$$\Lambda_t^{\pi} = \frac{\bar p_1(\tilde S_t)}{g_0(\tilde S_t)}
= \frac{\alpha\,
\exp\!\left(-\alpha\tilde{S}_t + \frac{1}{2}\alpha^2 v_1^2\right)
\Phi\!\left(\dfrac{\tilde{S}_t - \alpha v_1^2}{v_1}\right)}
{g_0(\tilde{S}_t)}.$$

\section{Pilot Calibration of the Prior Rate $\alpha$ and Scale $v_1$}
\label{app:pilot-calibration}

The mixture $\bar p_1(s;\alpha,v_1)$ of Appendix~\ref{app:tartakovsky-spivak}
is a function of two hyperparameters that Algorithm~\ref{alg:sr-detection}
takes as given. Section~\ref{sec:problem} shows that $\delta^2$ itself never
needs to be estimated -- the mixture already averages over it -- but
$(\alpha, v_1)$ still have to be fixed somehow before the mixture can be
evaluated at all. This appendix gives the plug-in construction used for
that in Section~\ref{sec:experiments}, together with what it does and does
not require.

\subsection*{Construction}

Fix a pool of $N_{\rm pool}$ observations standing in for $H_1$ -- either
genuine post-change data, if a validation or backtest window with a known
past change is available, or a domain-informed synthetic perturbation of
the training data (a hypothesised shift: a mean shift, a covariance
rotation, an inflated tail, etc.) when no real post-change instance has yet
been observed. Encode the pool through the \emph{same frozen} map
$\Phi(\cdot; \mathbf h_{\rm fixed})$ used online, then repeatedly draw a
window of size $w$ from the encoded pool (with replacement across draws) and
evaluate the closed-form RBF estimator on it, giving $N_{\rm pilot}$ draws
$$S^{(1)}, \ldots, S^{(N_{\rm pilot})} \overset{\text{i.i.d.}}{\sim}
\text{(approx.)}\ \tilde S \mid H_1.$$
Method-of-moments estimates follow directly from the model
$\tilde S \mid \delta^2 \sim \mathcal N(\delta^2, v_1^2)$,
$\delta^2 \sim \text{Exponential}(\alpha)$, whose marginal mean and variance
are $\mathbb E[\tilde S] = 1/\alpha$ and
$\mathrm{Var}(\tilde S) \approx v_1^2$ (the prior's own variance $1/\alpha^2$
is folded into $v_1^2$ here rather than fit separately, since the two
sources of spread are not identifiable from $N_{\rm pilot}$ scalar
draws alone):
$${\hat\delta^2 = \overline{S} = \frac{1}{N_{\rm pilot}}\sum_{i=1}^{N_{\rm pilot}} S^{(i)},
\qquad
\hat v_1 = \sqrt{\frac{1}{N_{\rm pilot}}\sum_{i=1}^{N_{\rm pilot}} \bigl(S^{(i)} - \overline S\bigr)^2},
\qquad
\hat\alpha = \frac{1}{\hat\delta^2}.}$$
These are the values substituted into $\Lambda_t^\pi$ and, consequently,
into the null simulation used for threshold calibration
(Appendix~\ref{app:calibration}), since that simulation also runs
$\Lambda_t^\pi$ with the fitted $(\hat\alpha, \hat v_1)$ on synthetic
$H_0$ paths.

\subsection*{What this does and does not require}

Only the \emph{shape} of the anticipated shift needs to be plausible, not
its exact post-change law: $\hat\alpha, \hat v_1$ enter $\Lambda_t^\pi$ only
through the mixture that gets averaged over $\delta^2$ in the first place,
so a pilot pool that gets the sign or rough magnitude of the effect right
but not its precise distribution still yields a serviceable detector --
Section~\ref{sec:experiments} reports one dataset (the chirality-flip pair)
where the true $\delta^2$ is small and pilot estimation faithfully
recovers this, rather than the discrepancy being an estimation failure.
What it does \emph{not} do is relieve the false-alarm ratio of any
dependence on $(\hat\alpha,\hat v_1)$: the null simulation of
Appendix~\ref{app:calibration} is run with the fitted mixture, so a badly
misspecified pilot pool changes \emph{which} threshold $A$ is required to
hit the target $\gamma$ (see the remark on non-universality of $A$ at the
end of that appendix) without changing the guarantee that whatever $A$ is
selected does hit it, since that calibration step is self-consistent by
construction -- it simulates under the same $(\hat\alpha,\hat v_1)$ it will
be deployed with. What degrades under misspecification is purely
statistical efficiency: a pilot pool that understates the eventual true
$\delta^2$ (large $\hat\alpha$) yields a mixture concentrated near small
effect sizes, which reacts more sluggishly once a larger shift actually
occurs, exactly as an underpowered study designed for too small an effect
size loses power against a larger one.

\subsection*{Statistics for synthetic data}

The table below provides main descriptive statistics for proper density estimation.

\begin{table}[h]
\centering
\begin{tabular}{lccccc}
\toprule
Experiment & $\hat\delta^2$ & $\log(1+A^*)$ & $\tau$ & $\hat\tau$ & delay \\
\midrule
Gaussian mixture rotation  & $0.207$ & $23.8$ & $100$ & $115$ & $15$ \\
Four clusters $\to$ one    & $0.070$ & $43.4$ & $100$ & $120$ & $20$ \\
Blob $\to$ Ring            & $0.267$ & $25.6$ & $100$ & $118$ & $18$ \\
\bottomrule
\end{tabular}
\caption{Pilot signal strength, calibrated threshold, and detection delay
(raw time steps) for the three experiments, all at the shared
$P_{\rm FA}\le5\%$/$175$-window budget.}
\label{tab:experiments}
\end{table}

\section{Numerical Stability of the Shiryaev--Roberts Recursion}
\label{app:numerical-stability}

The recursion of the main text,
$$R_t = (1+R_{t-1})\,\Lambda_t^{\pi}, \qquad R_0 = 0,$$
is exact but numerically fragile when evaluated in raw (non-logarithmic)
floating-point arithmetic, for two compounding reasons specific to this
construction.

\paragraph{Source of the fragility.}
First, $g_0$ is not available in closed form and is evaluated by a kernel
density estimate fitted to a Monte Carlo sample of $\tilde S$ under $P_0$
(Section~\ref{sec:problem}, Appendix~\ref{app:mmd-asymptotics}); like any
KDE, its estimated density decays to (numerically) zero outside the support
of the fitting sample. Whenever a realised $\tilde S_t$ falls in this
region -- which happens routinely once genuine post-change evidence
accumulates, since $\tilde S_t$ then concentrates near $\delta^2$, several
null standard deviations away from $0$ -- the ratio
$\Lambda_t^\pi = \bar p_1(\tilde S_t)/g_0(\tilde S_t)$ is a division by a
value at or below machine epsilon, producing a single-window likelihood
ratio of unbounded, uninformative magnitude rather than a merely large one.
Second, even with $\Lambda_t^\pi$ well behaved, the recursion multiplies
$t$ such factors together; under a sustained alternative $\log R_t$ grows
linearly in $t$ (since $\mathbb E_1[\log\Lambda_t^\pi] > 0$), so $R_t$
itself grows geometrically and exceeds the double-precision range
($R_t > e^{\approx 709}$) after only a few dozen windows of strong signal --
well inside the horizons used in Section~\ref{sec:problem}'s experiments --
silently overflowing to \texttt{inf} and terminating the recursion.

\paragraph{Fix 1: capping the per-window log-likelihood ratio.}
To keep any single window's evidence finite and commensurable with the rest
of the path regardless of how far $\tilde S_t$ strays from the fitted
support of $g_0$, we clip
$$\log\Lambda_t^\pi \ \leftarrow\
\operatorname{clip}\bigl(\log\bar p_1(\tilde S_t) - \log g_0(\tilde S_t),\ -L,\ L\bigr)$$
for a generous constant $L$ (we use $L=15$, i.e.\ a single window
contributes a likelihood ratio of at most $e^{15}\approx 3.3\times10^{6}$).
This leaves the recursion unaffected everywhere $|\log\Lambda_t^\pi| < L$ --
in particular throughout the null regime and the bulk of the alternative
regime -- and only truncates the rare windows where the KDE floor would
otherwise inject an artefactual, arbitrarily large jump. This is a routine
safeguard for any likelihood ratio built by dividing by an estimated
density.

\paragraph{Fix 2: an exact log-space reparametrisation.}
Even with $\Lambda_t^\pi$ capped, the running product across many windows
can still overflow. Rather than track $R_t$ directly, define
$$M_t := \log(1+R_t).$$
Substituting $1+R_{t-1} = e^{M_{t-1}}$ into the recursion gives
$$1+R_t = 1 + e^{M_{t-1}}\Lambda_t^\pi
= 1+\exp\!\bigl(M_{t-1}+\log\Lambda_t^\pi\bigr),$$
so that
$${M_t = \operatorname{softplus}\!\bigl(M_{t-1}+\log\Lambda_t^\pi\bigr),
\qquad M_0 = 0,} \qquad \operatorname{softplus}(x) := \log(1+e^x),$$
which is an \emph{exact} restatement of the recursion -- not an
approximation -- since $M_t$ is by definition $\log(1+R_t)$ at every $t$.
Evaluated via the standard stabilised form
$$\operatorname{softplus}(x) = \max(x,0) + \log\bigl(1+e^{-|x|}\bigr),$$
this never overflows for any finite argument, because the exponential is
always applied to a non-positive number. The threshold-crossing rule
$\tau_A = \inf\{t : R_t \geq A\}$ becomes the equivalent, equally exact rule
$$\tau_A = \inf\{t : M_t \geq \log(1+A)\}$$
on the chain $(M_t)$, so Algorithm~\ref{alg:sr-detection} and the threshold
calibration of Appendix~\ref{app:calibration} -- which operate on the
$[0,1)$-scaled chain $\Pi_n = R_n/(1+R_n) = 1-e^{-M_n}$, itself unaffected
by the reparametrisation -- carry over unchanged; only the internal
representation used to accumulate evidence changes, and comparing to a
threshold expressed on the same log scale requires no re-derivation.

\section{Threshold Calibration: Two Free-Boundary Problems}
\label{app:calibration}

\subsection*{Preliminaries: the chain $(\Pi_n)$}

Write $\Pi_n = R_n/(1+R_n)$ for the SR statistic on the $[0,1)$ scale, and
recall the recursion $R_n = (1+R_{n-1})\Lambda_n^\pi$. In terms of $\Pi_n$
this reads
$$\Pi_n = \frac{(1+R_{n-1})\Lambda_n^\pi}{1+(1+R_{n-1})\Lambda_n^\pi}
= \frac{\dfrac{\Pi_{n-1}}{1-\Pi_{n-1}}\Lambda_n^\pi}
       {1+\dfrac{\Pi_{n-1}}{1-\Pi_{n-1}}\Lambda_n^\pi}
=: T(\Pi_{n-1}, \Lambda_n^\pi),$$
so $(\Pi_n)_{n\ge0}$ is a time-homogeneous Markov chain on $[0,1)$ driven by
the i.i.d.\ sequence $\Lambda_1^\pi, \Lambda_2^\pi,\ldots$ -- the window-indexed,
discrete-time analogue of the posterior probability process $(\pi_t)_{t\ge0}$
that \citet[\S22.0, \S24.1]{Shiryaev_Peskir} construct from the continuously
observed Wiener or Poisson process (compare $T$ above with their update rule
$\pi_t=\varphi_t/(1+\varphi_t)$, eqs.\ (22.0.8)--(22.0.9), (24.1.8)). Since
$(\Pi_n)$ moves only at integer window steps, by a random multiplicative jump
of continuous distribution at \emph{every} step, it never moves continuously
between updates: for any level $b\in(0,1)$, $\Pi_n$ almost surely
\emph{overshoots} $b$ rather than landing on it. This single fact governs the
boundary condition in both approaches below: wherever a differential
(``smooth-fit'') condition would be imposed for a diffusion
(\citealt[Thm.~22.1]{Shiryaev_Peskir}) or could still sometimes be imposed for a
jump-diffusion (their Poisson case, Theorem 24.1(i)), it is not meaningful
here, and only \emph{continuous fit} -- plain value matching -- survives, in
line with their Theorem 24.1(ii)--(iii) and the discussion of jump entrances
in \S24.3 (Figure VI.8).

\subsection*{Approach A -- Bayes risk minimisation}

Following the reduction of the Bayes risk (their (22.0.4)) to the optimal
stopping problem (their (22.0.5)) for $(\pi_t)$, the risk
$\mathcal{R}(\tau) = c\,\mathbb{E}[(\tau-\theta)^+]+\mathbb{P}(\tau<\theta)$,
re-expressed on the window scale using the SR statistic in place of the
proper-prior posterior, becomes an optimal stopping problem for $(\Pi_n)$:
$$V(x) = \inf_{n\ge0}\mathbb{E}_x\!\left[1-\Pi_n + c\sum_{k=0}^{n-1}\Pi_k\right],
\qquad \mathbb{P}_x(\Pi_0=x)=1.$$
Since $(\Pi_n)$ is a discrete-time Markov chain, the role of the infinitesimal
generator is played by the \emph{one-step} operator (cf.\ their (22.1.2) and
(24.1.16)):
$$(\mathbb{L}f)(x) = \mathbb{E}\bigl[f(T(x,\Lambda^\pi))\bigr] - f(x),
\qquad \Lambda^\pi \stackrel{d}{=}\Lambda_1^\pi.$$
Guessing (and, as in their Theorems 22.1/24.1, subsequently verifying via a
standard martingale argument) that it is optimal to stop the first time
$\Pi_n$ reaches or exceeds some level $b$,
$$\tau_A = \inf\{n\ge0 : \Pi_n\ge b\} = \inf\{n\ge0: R_n\ge A\},
\qquad A = \frac{b}{1-b},$$
the value function must satisfy, on the continuation region $\{x<b\}$,
$$(\mathbb{L}V)(x) = -cx, \qquad 0\le x<b,$$
and $V(x)=1-x$ on the stopping region $\{x\ge b\}$. By the preliminary
discussion above, only continuous fit applies at the unknown boundary:
$${(\mathbb{L}V)(x) = -cx \ \ (0\le x<b), \qquad
V(x) = 1-x \ \ (b \le x < 1), \qquad
V(b{-}) = 1-b}$$
Unlike the Wiener case (solved in closed form via the integrating factor,
their (22.1.10)--(22.1.12)) or the Poisson case (solved via the step function
(24.1.24)--(24.1.33), requiring the Gauss hypergeometric function once more
than one region is involved), the one-step operator $\mathbb{L}$ here
involves an expectation over $\Lambda^\pi$, whose density is the ratio of the
closed-form mixture $\bar{p}_1$ to the non-closed-form null density $g_0$
(Appendices~\ref{app:tartakovsky-spivak}, \ref{app:mmd-asymptotics}), so no
elementary solution to this system is available; $(V,b^*(c))$ is obtained
numerically by value iteration of the fixed-point equation above on a grid
in $[0,1)$. Repeating this for a range of $c$ traces out the achievable
delay/false-alarm frontier.

\subsection*{Approach B -- ARL-constrained calibration}

Here no cost $c$ is introduced. Instead, fix a target average run length to
false alarm $\gamma>0$ and let
$$U(x) = \mathbb{E}_x^{\infty}[\tau_A], \qquad \tau_A = \inf\{n\ge0:\Pi_n\ge b\},$$
be the expected number of further windows to threshold-crossing under $P_\infty$
(no change), starting from $\Pi_0=x$. By the same one-step decomposition as
above (now with running cost $\equiv 1$ per window rather than $c\Pi_k$, and no
minimisation -- the rule is already fixed as a threshold crossing, so this is a
linear first-passage problem, not a variational one),
$${U(x) = 1 + \mathbb{E}\bigl[U(T(x,\Lambda^\pi))\bigr] \ \ (0\le x<b), \qquad
U(x) = 0 \ \ (b\le x<1)}$$
again with $U$ merely continuous (value-matching, $U(b{-})=0$) rather than
smooth at $b$, for the same overshoot reason as above. The threshold is then
the value of $b$ (equivalently $A=b/(1-b)$) solving
$$U(0) = \gamma.$$
As with Approach A, $\Lambda^\pi$ has no elementary distribution, so $U$ and
$b(\gamma)$ are obtained numerically: simulate $(R_n)$ under $P_\infty$ for a
grid of candidate thresholds $A$, estimate $\mathbb{E}_\infty[\tau_A]$ by the
empirical mean first-passage time over repeated simulated paths, and select
$A$ matching the target $\gamma$ -- exactly as $g_0$ itself is evaluated by
simulation in Section~\ref{sec:null_dist}. By the \citet{Pollak1985} theorem, the
resulting rule is then asymptotically minimax-optimal as $\gamma\to\infty$,
with no cost parameter, and hence no implicit assumption about the frequency
of changes, ever required.

\subsection*{Remark: why $A$ itself is not universal across problems}

The exact identity underlying both approaches is worth isolating on its
own. Since $R_n$ is a $P_\infty$-martingale with $R_0 = 0$ regardless of
which mixture built $\Lambda^\pi$ -- $\mathbb E_\infty[\Lambda^\pi] =
\int \bar p_1(s)\,ds = 1$ for \emph{any} valid mixture density
$\bar p_1(\cdot;\alpha,v_1)$, exactly as $\mathbb E_0[p_1(X)/p_0(X)] =
\int p_1 = 1$ in the classical, known-alternative case -- optional
stopping applied at $\tau_A$ gives, \emph{exactly} rather than
asymptotically,
$$\mathbb{E}_\infty[\tau_A] = A + \mathbb{E}_\infty\bigl[R_{\tau_A} - A\bigr],$$
the second term being the mean \emph{overshoot} of $R_n$ past the boundary
at the moment it is crossed. Only the overshoot depends on which
$(\hat\alpha,\hat v_1)$ built $\Lambda_n^\pi$ (Appendix~\ref{app:pilot-calibration}):
when it is small relative to $A$ -- light-tailed $\Lambda^\pi$, large $A$
-- the familiar approximation $A \approx \gamma$ is recovered essentially
independently of the alternative, which is the precise sense in which
Pollak's asymptotic optimality above holds for \emph{any} alternative as
$\gamma\to\infty$. At the moderate, realistic $\gamma$ used in
Section~\ref{sec:experiments}, the overshoot is not negligible:
$\Lambda_n^\pi$ can jump by a large factor whenever a null $\tilde S_n$
lands where the fitted mixture $\bar p_1$ disagrees sharply with the true
null density $g_0$ in its tail -- the same mechanism flagged in
Appendix~\ref{app:numerical-stability} -- and the size of that
disagreement is governed by $(\hat\alpha,\hat v_1)$. Different pilot fits
therefore give different overshoot distributions and hence different $A$
for the \emph{same} target $\gamma$, even though $\gamma$ itself, being
exactly the quantity Approach B calibrates against, comes out identical
by construction across every dataset in Section~\ref{sec:experiments}.

\section{The Full Detection Procedure}
\label{app:algorithm}

This appendix collects the pieces derived through the main text --- the
frozen PF-ODE encoder (Appendix~\ref{app:bijection}), the closed-form RBF
MMD$^2$, the Tartakovsky--Spivak mixture likelihood ratio
(Appendix~\ref{app:tartakovsky-spivak}, pilot-calibrated per
Appendix~\ref{app:pilot-calibration}), and the Pollak-calibrated
Shiryaev--Roberts threshold (Appendix~\ref{app:calibration}) --- into the
single online procedure referenced throughout as
Algorithm~\ref{alg:sr-detection}, with a one-time offline stage and a
per-window online loop.

\begin{algorithm}[H]
\SetAlgoLined
\DontPrintSemicolon
\caption{Sequential PF-ODE / Shiryaev--Roberts changepoint detection}
\label{alg:sr-detection}

\KwIn{Burn-in sample $X_1,\ldots,X_{T_{\rm burn}}$; window size $w$; RBF bandwidth $\sigma$; prior rate $\alpha$; target average run length $\gamma$}
\KwOut{Alarm time $\hat\theta$}

\tcc{Offline}
Train $\hat\varepsilon_\theta$, $\mathrm{Enc}$ jointly on $X_1,\ldots,X_{T_{\rm burn}}$ (DSM loss)\;
$\mathbf h_{\rm fixed} \leftarrow \mathrm{Enc}(X_1,\ldots,X_{T_{\rm burn}})$\;
Tabulate $g_0$ by Monte Carlo under $\mathcal N(0,I)$\;
Calibrate $A$: solve $U(0)=\gamma$ (Appendix~\ref{app:calibration}, Approach B)\;
$n \leftarrow 0$;\ $R_0 \leftarrow 0$\;

\BlankLine
\tcc{Online}
\While{no alarm}{
    $n \leftarrow n+1$\;
    Read window $X_{(n-1)w+1},\ldots,X_{nw}$\;
    $Z_i \leftarrow \Phi(X_i;\mathbf h_{\rm fixed})$ for $i=1,\ldots,w$\;
    $\widehat{\mathrm{MMD}}^2_n \leftarrow$ closed-form RBF estimator on $\{Z_i\}$\;
    $\tilde S_n \leftarrow \widehat{\mathrm{MMD}}^2_n - \mu_0$\;
    $\Lambda_n^\pi \leftarrow \bar p_1(\tilde S_n)/g_0(\tilde S_n)$\;
    $R_n \leftarrow (1+R_{n-1})\,\Lambda_n^\pi$\;
    \eIf{$R_n \ge A$}{
        \Return $\hat\theta \leftarrow (n-1)w+1$\;
    }{
        continue\;
    }
}
\end{algorithm}
\hspace{5mm}
\begin{remark}[Design choices]
Several steps deserve comment. The context $\mathbf h_{\rm fixed}$ is frozen
once, after burn-in, and never updated: this is what makes
$\Phi(\cdot;\mathbf h_{\rm fixed})$ a fixed, well-defined bijection
(Appendix~\ref{app:bijection}) rather than a moving target, and why encoding
uses the deterministic PF-ODE map rather than a DDPM-style stochastic
sampler. The ``(DSM loss)'' step above fits $\hat\varepsilon_\theta$ and
$\mathrm{Enc}$ jointly to the conditional Stein score that this drift
requires; see Appendix~\ref{app:score-training} for how that gradient is
defined and estimated. Windows are non-overlapping, which is what makes
$\tilde S_1, \tilde S_2,\ldots$ i.i.d.\ within a regime -- the condition the
SR recursion needs: because $\mathbf h_{\rm fixed}$ is frozen, every $X_i$ in
the online phase is drawn i.i.d.\ from $p_0(\cdot\mid\mathbf h_{\rm fixed})$
(or $p_1(\cdot\mid\mathbf h_{\rm fixed})$ after the change), so non-overlapping
windows partition this i.i.d.\ sequence into disjoint blocks; disjoint blocks
of independent variables are themselves independent, and equal-sized blocks of
the same i.i.d.\ law are identically distributed, so passing each block
through the same fixed map $\tilde S_n = \widehat{\mathrm{MMD}}^2_n - \mu_0$
yields an i.i.d.\ sequence $\tilde S_1, \tilde S_2,\ldots$. \emph{Overlapping}
(sliding) windows break this disjointness and hence the i.i.d.\ property --
exactly the caveat noted in Section~\ref{sec:experiments} for the sliding
construction actually used there. The null density $g_0$
is tabulated once offline, since it depends only on $(\sigma,w,d)$ and not on
the incoming data stream; $\bar p_1$ is the closed-form Tartakovsky--Spivak
mixture of Appendix~\ref{app:tartakovsky-spivak}. The threshold $A$ is
calibrated via Approach B (Appendix~\ref{app:calibration}) rather than
Approach A: it requires only an operational ARL budget $\gamma$, with no cost
parameter $c$ and no assumption about how often changes occur.
\end{remark}

\hspace{5mm}
\begin{remark}[Cost and scope]
Encoding each $X_i$ requires solving the PF-ODE and is the dominant cost;
everything downstream -- $\widehat{\mathrm{MMD}}^2_n$, $\Lambda_n^\pi$,
$R_n$ -- is closed-form arithmetic on $w$ numbers. Algorithm~\ref{alg:sr-detection}
targets a single changepoint, matching the model of Section~\ref{sec:problem}
($\theta$ singular); for continuous monitoring across multiple regime changes,
the natural extension is to restart at the alarm (re-estimate
$\mathbf h_{\rm fixed}$ from a fresh burn-in window, reset $R\leftarrow0$),
which we do not analyse here.
\end{remark}

\section{Training the PF-ODE Encoder: Denoising Score Matching and the Stein Identity}
\label{app:score-training}

Appendix~\ref{app:bijection} treats $\Phi(\cdot;\mathbf h)$ as a black box: any injective PF-ODE flow suffices for the change-detection argument, regardless of how its drift is obtained. This appendix opens that black box, and is the point in the paper where machine learning actually does the work: every quantity derived in Sections~\ref{sec:mmd}--\ref{sec:baselines} and Appendices~\ref{app:BC-derivation}--\ref{app:calibration} is closed-form or Monte-Carlo-tabulated classical probability theory, applied on top of a single learned object -- the conditional \emph{Stein score} of the noised pre-change density. This appendix defines that object, explains why the change-point problem of Section~\ref{sec:problem} cannot supply it any other way, and derives the training objective (the ``DSM loss'' of Algorithm~\ref{alg:sr-detection}) used to fit a network to it.

\subsection*{The forward noising SDE and the score that generates its reverse}

Fix the frozen context $\mathbf h=\mathbf h_{\rm fixed}$ and, conditionally on it, augment each pre-change observation $X_0\sim p_0(\cdot\mid\mathbf h)$ with an auxiliary diffusion time $\tau\in[0,T]$\footnote{We reserve $t$, as throughout the rest of the paper, for the outer sequential/window index of Section~\ref{sec:problem}; discretised, $\tau$ becomes the step index of Appendix~\ref{app:architectures}'s noise schedule and the DDIM/PF-ODE step counts of Table~\ref{tab:architectures}.} running a forward It\^o SDE
$$dX_\tau = f(X_\tau,\tau)\,d\tau + g(\tau)\,dW_\tau, \qquad X_0\sim p_0(\cdot\mid\mathbf h),$$
whose marginal density $p_\tau(\cdot\mid\mathbf h)$ interpolates between the unknown $p_0(\cdot\mid\mathbf h)$ at $\tau=0$ and, for a suitably designed $(f,g)$, a law indistinguishable from $\mathcal N(0,I)$ at $\tau=T$. \citet{Anderson1982} shows this process can be run backward in time with the \emph{same} marginals $p_\tau(\cdot\mid\mathbf h)$ at every $\tau$, via the reverse-time SDE
$$dX_\tau = \bigl[f(X_\tau,\tau) - g(\tau)^2\nabla_x\log p_\tau(X_\tau\mid\mathbf h)\bigr]\,d\tau + g(\tau)\,d\bar W_\tau,$$
and \citet{SongEtAl2021SDE} (Thm.~1) show a deterministic ODE sharing the very same marginals exists, obtained by discarding the stochastic term while halving the score coefficient:
$${dX_\tau = \Bigl[f(X_\tau,\tau) - \tfrac12 g(\tau)^2\,\nabla_x\log p_\tau(X_\tau\mid\mathbf h)\Bigr]\,d\tau.}$$
This is exactly the drift in the Remark of Appendix~\ref{app:bijection}, whose Lipschitzness in $x$ is what makes the flow map $\Phi(\cdot;\mathbf h)$ -- solving this ODE from $\tau=0$ to $\tau=T$ -- injective. The one term in that drift not given in closed form is $\nabla_x\log p_\tau(x\mid\mathbf h)$, the \emph{Stein score} of the noised conditional density at time $\tau$: since $p_\tau(\cdot\mid\mathbf h)$ is the law of $p_0(\cdot\mid\mathbf h)$ pushed through $\tau$ steps of noising, it inherits the same lack of closed form as $p_0$ itself (Section~\ref{sec:problem}) at every $\tau<T$ -- diffusing an intractable density does not make it tractable, it only smooths it towards a known Gaussian endpoint. Fitting a network to this score, for every $\tau$ and every $\mathbf h$ the encoder will see, is the entire learning problem this paper relies on.

Specialising to the variance-preserving schedule used throughout (the linear $\beta(\tau)$ schedule of Appendix~\ref{app:architectures}), $f(x,\tau)=-\tfrac12\beta(\tau)x$, $g(\tau)=\sqrt{\beta(\tau)}$, so that the forward SDE has the closed-form Gaussian transition
$${X_\tau \mid X_0=x_0 \;\sim\; \mathcal N\bigl(\mu_\tau x_0,\ \sigma_\tau^2 I\bigr), \qquad \mu_\tau=\exp\Bigl(-\tfrac12\!\int_0^\tau\!\beta(u)\,du\Bigr), \quad \sigma_\tau^2=1-\mu_\tau^2,}$$
matching the $\mu_t,\sigma_t$ of Appendix~\ref{app:architectures}'s $v$-prediction target and recovering $\mu_\tau\to0,\ \sigma_\tau\to1$ -- the $\mathcal N(0,I)$ endpoint Section~\ref{sec:problem} maps pre-change data onto -- as $\tau\to T$.

\subsection*{Why ``Stein'' score, and why it cannot be matched directly}

For a smooth density $p$ with sufficient tail decay and any smooth $\phi:\mathbb R^d\to\mathbb R^d$ vanishing at infinity, integration by parts gives \citet{Stein1981}'s identity
$$\mathbb E_{X\sim p}\bigl[\nabla_x\log p(X)^\top\phi(X) + \nabla_x\!\cdot\phi(X)\bigr] = 0,$$
which is why $\nabla_x\log p$ is called the \emph{Stein score} of $p$ -- to distinguish it from the statistical score $\nabla_\theta\log p_\theta$ with respect to a parameter -- and why Stein's identity, rather than direct comparison to a known target, is the mechanism by which a network can be fit to it \citep{LiuLeeJordan2016,SongErmon2019}. Applying it with $\phi=s_\theta$, the naive objective $\mathbb E_p\|s_\theta(x)-\nabla_x\log p(x)\|^2$ -- which needs the unknown target $\nabla_x\log p$ itself -- is shown by \citet{Hyvarinen2005} to equal, up to an additive constant not depending on $\theta$,
$$J(\theta) = \mathbb E_{X\sim p}\Bigl[\operatorname{tr}\bigl(\nabla_x s_\theta(X)\bigr) + \tfrac12\|s_\theta(X)\|^2\Bigr],$$
an objective that needs only samples from $p$ -- exactly what the pre-change burn-in of Section~\ref{sec:problem} supplies -- and never the density itself. This \emph{implicit} score-matching objective is, however, not what Algorithm~\ref{alg:sr-detection} trains with: $\operatorname{tr}(\nabla_x s_\theta)$ requires the full Jacobian of the network, one backward pass per output coordinate, which is infeasible at the dimension used in Section~\ref{sec:real-data} ($d=256$).

\subsection*{Denoising score matching: replacing the marginal score with a tractable conditional one}

\citet{Vincent2011} sidesteps the Jacobian trace entirely by matching, instead of the marginal score $\nabla_{x_\tau}\log p_\tau(x_\tau\mid\mathbf h)$, the score of the \emph{conditional} transition kernel $p_{\tau\mid0}(x_\tau\mid x_0)$ derived above, which is Gaussian and hence known in closed form:
$$\nabla_{x_\tau}\log p_{\tau\mid0}(x_\tau\mid x_0) = -\frac{x_\tau-\mu_\tau x_0}{\sigma_\tau^2} = -\frac{\eta}{\sigma_\tau}, \qquad X_\tau=\mu_\tau X_0+\sigma_\tau\eta,\ \ \eta\sim\mathcal N(0,I).$$
These two targets share the same minimiser because the marginal score is itself the posterior mean of the conditional one. Writing $p_\tau(x_\tau\mid\mathbf h)=\int p_{\tau\mid0}(x_\tau\mid x_0)\,p_0(x_0\mid\mathbf h)\,dx_0$ and differentiating under the integral sign,
$$\nabla_{x_\tau}p_\tau(x_\tau\mid\mathbf h) = \int p_{\tau\mid0}(x_\tau\mid x_0)\,\nabla_{x_\tau}\log p_{\tau\mid0}(x_\tau\mid x_0)\,p_0(x_0\mid\mathbf h)\,dx_0,$$
and dividing both sides by $p_\tau(x_\tau\mid\mathbf h)$ turns the integrand's first two factors, by Bayes' rule, into the posterior $p_{0\mid\tau}(x_0\mid x_\tau,\mathbf h)$, giving the identity underlying the whole construction:
$${\nabla_{x_\tau}\log p_\tau(x_\tau\mid\mathbf h) = \mathbb E_{X_0\mid X_\tau=x_\tau,\,\mathbf h}\bigl[\nabla_{x_\tau}\log p_{\tau\mid0}(x_\tau\mid X_0)\bigr].}$$
Expanding both squared objectives and using this identity to equate their cross-terms (the tower property turns an expectation of $s_\theta^\top\nabla\log p_\tau$ under the marginal into the same expectation of $s_\theta^\top\nabla\log p_{\tau\mid0}$ under the joint) shows the two share the same $\theta$-independent remainder, and hence the same minimiser:
$$\argminC_\theta\ \mathbb E_{X_\tau\mid\mathbf h}\bigl\|s_\theta(X_\tau,\tau;\mathbf h)-\nabla_{x_\tau}\log p_\tau(X_\tau\mid\mathbf h)\bigr\|^2 = \argminC_\theta\ \mathbb E_{X_0,X_\tau\mid\mathbf h}\bigl\|s_\theta(X_\tau,\tau;\mathbf h)-\nabla_{x_\tau}\log p_{\tau\mid0}(X_\tau\mid X_0)\bigr\|^2.$$
Substituting the closed form of the conditional score and reparametrising $s_\theta=-\hat\varepsilon_\theta/\sigma_\tau$ turns score matching into noise prediction -- the DSM loss cited by Algorithm~\ref{alg:sr-detection}:
$${\mathcal L_{\rm DSM}(\theta) = \mathbb E_{\tau\sim\mathcal U[0,T],\;X_0\sim p_0(\cdot\mid\mathbf h),\;\eta\sim\mathcal N(0,I)}\Bigl[\lambda(\tau)\,\bigl\|\hat\varepsilon_\theta(\mu_\tau X_0+\sigma_\tau\eta,\ \tau;\ \mathbf h) - \eta\bigr\|^2\Bigr],}$$
with $\lambda(\tau)$ a per-$\tau$ reweighting (the $v$-prediction of Appendix~\ref{app:architectures} is one such reweighting, chosen for numerical conditioning rather than a different objective in substance) and no Jacobian trace ever computed: $\eta$ is simulated directly, and $\hat\varepsilon_\theta$ is trained by ordinary backpropagation of a mean-squared error.

\subsection*{Where the context enters, and what is actually being learned}

Algorithm~\ref{alg:sr-detection} trains $\hat\varepsilon_\theta$ and $\mathrm{Enc}$ \emph{jointly} against $\mathcal L_{\rm DSM}$: the context $\mathbf h_{t-1}=\mathrm{Enc}(X_1,\ldots,X_{t-1})$ is not a hand-designed summary statistic but is itself learned, end-to-end, to be whatever function of the history makes the conditional score $\nabla_x\log p_\tau(x\mid\mathbf h)$ most accurately estimable by $\hat\varepsilon_\theta$ jointly across every $\tau$. This is the one place in the entire pipeline where a network is fit to data at all: the closed-form RBF $\mathrm{MMD}^2$ (Section~\ref{sec:mmd}), the Tartakovsky--Spivak mixture (Appendix~\ref{app:tartakovsky-spivak}), and the Pollak-calibrated threshold (Appendix~\ref{app:calibration}) are all exact or Monte-Carlo-tabulated probability theory built on top of the single frozen map $\Phi(\cdot;\mathbf h_{\rm fixed})$ that this training produces. Machine learning's role in solving the change-point problem of Section~\ref{sec:problem} is exactly to supply this one otherwise-inaccessible quantity -- the conditional Stein score of a density family with no closed form -- which classical likelihood-ratio methods needed in exactly the same place (the ratio $p_1/p_0$ of Section~\ref{sec:problem}) and had no way to obtain.

\begin{remark}[The failure mode already seen in Section~\ref{sec:real-data}]
The identity above only guarantees that $\hat\varepsilon_\theta$ recovers the true score \emph{at} $\mathcal L_{\rm DSM}$'s global minimum. A network stopped short of it -- as in the reduced MNIST recipe of Section~\ref{sec:real-data} (Table~\ref{tab:architectures}) -- yields a biased score estimate at every $\tau$, and the resulting empirical pushforward $\Phi(\cdot;\mathbf h_{\rm fixed})_\#p_0$ is correspondingly not exactly $\mathcal N(0,I)$: this is the same shortfall reported there as a per-dimension latent variance of $0.845$ rather than $1$. Every guarantee derived from Section~\ref{sec:mmd} onward is conditional on this training objective having converged; nothing in the downstream MMD/Shiryaev--Roberts machinery can compensate for it not having done so.
\end{remark}

\section{Model Architectures Used in the Experiments}
\label{app:architectures}

Section~\ref{sec:problem} leaves $\mathrm{Enc}$ and the denoiser
$\hat\varepsilon_\theta$ unspecified beyond the causality of the former
and the injectivity of the resulting PF-ODE (Appendix~\ref{app:bijection});
neither claim depends on the concrete networks below. This appendix fixes
the two instantiations actually used, in Section~\ref{sec:experiments} and
Section~\ref{sec:real-data} respectively -- both trained once, with no
architecture search.

\paragraph{History encoder.} A single-layer GRU~\citep{ChoEtAl2014GRU} in
both cases. For the $\mathbb R^2$ pairs it consumes each raw observation
$X_i\in\mathbb R^2$ directly; for MNIST each $16\times16$ frame is first
mapped to a $64$-dimensional embedding by a two-layer MLP
($256\to256\to64$, SiLU~\citep{ElfwingEtAl2018SiLU}) before the GRU. In
both cases $\mathbf h_{t-1}$ is the GRU's hidden state after processing
$X_1,\ldots,X_{t-1}$ (zero at $t=0$), which is exactly the causal,
one-step-shifted encoding Section~\ref{sec:problem} requires of
$\mathrm{Enc}$.

\paragraph{Denoiser.} Both variants are pre-norm residual
stacks~\citep{BaKirosHinton2016LayerNorm} predicting
$v=\mu_t\varepsilon-\sigma_t x_0$~\citep{SalimansHo2022} rather than
$\varepsilon$ directly, so
that $x_0$ and $\varepsilon$ can be recovered at inference without dividing
by $\mu_t$ -- this is what keeps the DDIM inversion~\citep{SongEtAl2021DDIM} of
Section~\ref{sec:problem} well-conditioned at the large-$t$, small-$\mu_t$
end of the schedule, and is the network-level analogue of the
log-space reparametrization Appendix~\ref{app:numerical-stability} uses to
keep the SR recursion itself from overflowing. They differ only in how the
context $\mathbf h_{t-1}$ enters the residual block:
\begin{itemize}
    \item \emph{2D pairs (FiLM):} the context is projected once to a
    per-block $(\mathrm{scale}, \mathrm{shift})$ pair, applied after the
    block's own linear-SiLU transform: $h \leftarrow h +
    \big[\mathrm{SiLU}(\mathrm{Linear}(\mathrm{LayerNorm}(h)))\odot(1+\mathrm{scale})+\mathrm{shift}\big]$.
    \item \emph{MNIST (concatenation):} $x$, a sinusoidal diffusion-step
    embedding, and $\mathbf h_{t-1}$ are concatenated and projected once to
    the working width before the residual stack (no per-block FiLM).
    Section~\ref{sec:real-data}'s illustration additionally zeroes the
    context with probability $0.10$ during training (classifier-free
    guidance), so the same frozen context that drives detection can also
    drive the conditional generation of Figure~\ref{fig:mnist}; the
    corrected recipe used for Table~\ref{tab:baselines}
    (Table~\ref{tab:architectures}) drops this dropout entirely; see
    Section~\ref{sec:baselines} for why.
\end{itemize}

\begin{table}[h]
\centering
\small
\begin{tabular}{lcc}
\toprule
& \textbf{2D pairs (\S\ref{sec:experiments})} & \textbf{MNIST (\S\ref{sec:real-data})} \\
\midrule
observation dim $d$         & $2$                           & $256$ ($16\times16$) \\
history encoder              & GRU, hidden $64$              & MLP embed ($64$) $\to$ GRU, hidden $512$\textsuperscript{$\dagger$} \\
denoiser conditioning        & FiLM                          & concatenation, CFG dropout\textsuperscript{$\dagger$} \\
denoiser width / \# blocks   & $128$ / $4$                   & $256$ / $8$\textsuperscript{$\dagger$} \\
diffusion steps $T$          & $500$                         & $1000$ \\
$\beta$ schedule              & linear, $10^{-4}\!\to\!0.02$  & linear, $10^{-4}\!\to\!0.02$ \\
training series length        & $256$                         & $125$ \\
series per step               & $32$                          & $64$ \\
training steps                & $3{,}000$                     & $100{,}000$\textsuperscript{$\dagger$} \\
LR schedule                   & cosine, $3\!\times\!10^{-4}\!\to\!3\!\times\!10^{-6}$, $100$-step warmup
                               & cosine, $10^{-4}\!\to\!3\!\times\!10^{-6}$, $1000$-step warmup\textsuperscript{$\dagger$} \\
gradient clip (max norm)      & --                            & $1.0$\textsuperscript{$\dagger$} \\
EMA decay                     & $0.995$                       & $0.999$ \\
DDIM / PF-ODE steps            & $499$ (full)                  & $999$ (full)\textsuperscript{$\dagger$} \\
\bottomrule
\end{tabular}
\caption{Architecture and training hyperparameters for the two model
families used throughout the paper. $\dagger$: the MNIST column shows the
corrected recipe used to produce Table~\ref{tab:baselines} and the
convergence study of Section~\ref{sec:baselines} -- hidden width $512$,
$8$ denoiser blocks, learning rate $10^{-4}$ with gradient-norm clipping
at $1$, no CFG dropout, $120{,}000$ steps, full $999$-step DDIM budget.
Section~\ref{sec:real-data}'s single-trajectory illustration
(Figure~\ref{fig:mnist}) instead uses the original, smaller recipe
disclosed there for tractability: hidden width $256$, $4$ denoiser
blocks, learning rate $3\times10^{-4}$ (unclipped), no CFG dropout,
$6{,}000$ steps, and only $100$ DDIM steps; see the caveat stated there
and the discussion in Section~\ref{sec:baselines} for why the two
recipes diverged and what the correction changed.}
\label{tab:architectures}
\end{table}

\section{Baseline Detectors and Equal-Footing Evaluation}
\label{app:baselines}

This appendix specifies the three baselines compared against in the main
text (end of Section~\ref{sec:real-data}) and the Monte Carlo protocol
used to produce Table~\ref{tab:baselines}.

\subsection*{The equal-footing rule}
Section~\ref{sec:problem} assumes only that a burn-in segment of the
monitored series is known to be drawn from $p_0$; nothing else about
$p_0$ or $p_1$ is assumed known. Every baseline below is held to exactly
this standard: none of them is ever given a fresh draw from the true
$p_0$ generator for null calibration. Where a baseline needs a reference
pool or a covariance estimate, it is built from -- or bootstrapped from --
the burn-in segment of the \emph{same} series it monitors, i.e.\ exactly
the same $w$ observations "ours" itself would have access to at that
point. The one exception, applied identically to "ours" as to A and B, is
\emph{pilot calibration}: choosing $(\hat\alpha,\hat v_1)$
(Appendix~\ref{app:pilot-calibration}) still uses an idealised $p_1$
pilot pool throughout the paper, since pilot calibration only sets
detection \emph{speed}, not false-alarm validity, and this idealisation
is unchanged from how "ours" is evaluated everywhere else.

\subsection*{Baseline A: raw two-sample $\widehat{\mathrm{MMD}}^2$, online self-normalised}
Baseline A removes the PF-ODE encoder entirely and computes the
two-sample $\widehat{\mathrm{MMD}}^2$ (RBF kernel, median-heuristic
bandwidth from the burn-in segment) between the current window and a
fixed reference pool -- the burn-in segment itself. Unlike the statistic
against the known $\mathcal N(0,I)$ reference that "ours" enjoys, this
raw-space two-sample statistic's null law has \emph{no} known closed form, not even
asymptotically: Appendix~\ref{app:mmd-asymptotics}'s own derivation shows
$\widehat{\mathrm{MMD}}^2$ is generically a weighted sum of centred
$\chi^2(1)$ variables, not Gaussian, even when the reference law
\emph{is} $\mathcal N(0,I)$. Rather than fitting a KDE to data the problem
setup does not grant access to, baseline A standardises online: at each
step $t$ the raw statistic $u_t$ is centred and scaled by the mean/std of
a trailing causal buffer of its own $60$ most recent values (seeded from
a bootstrap of the burn-in segment), and the standardised value is fed
into the \emph{same} Tartakovsky--Spivak mixture likelihood
ratio of Appendix~\ref{app:tartakovsky-spivak} used everywhere else in the paper, just with
$g_0=\varphi$ (the standard normal density) in place of the exact/KDE
null. This is an explicit, honest approximation -- exactly the price of
not having the injectivity theorem -- rather than one hidden inside a
density estimate fit to data that would not really be available.
Threshold calibration bootstraps synthetic null paths by resampling
individual burn-in observations with replacement (a valid bootstrap of
$p_0$ under the i.i.d.\ model), never drawing from the true generator.

\subsection*{Baseline B: mean-embedding on the same latents}
Baseline B keeps the PF-ODE encoder but replaces
$\widehat{\mathrm{MMD}}^2$ with the cruder statistic
$\|\overline Z_w\|^2$, the squared norm of the window mean of the
\emph{same} latents "ours" uses. Under $H_0$, $Z\sim\mathcal N(0,I_d)$
exactly (Lemma~\ref{lem:pushforward-injective}), so
$\overline Z_w\sim\mathcal N(0,I_d/w)$ and $w\|\overline Z_w\|^2\sim
\chi^2_d$ \emph{exactly} -- the same theorem "ours" relies on, applied to
a linear-kernel analogue of Table~\ref{tab:kernels}'s row-by-row
comparison. B isolates what the mixture-SR machinery buys beyond a naive
mean shift once the encoder has already done its job; because none of the
three synthetic pairs are constructed as pure mean shifts (rotation,
mode-count change, and a radial redistribution respectively all leave the
mean approximately unchanged), B is expected -- and, in
Table~\ref{tab:baselines}, found -- to be nearly blind to all of them.

\subsection*{Baseline C: raw Hotelling's $T^2$ + CUSUM}
Baseline C uses no encoder and no MMD at all: it computes Hotelling's
$T^2$~\citep{Hotelling1931}, the (diagonal-covariance) squared
Mahalanobis distance of the window mean from a burn-in-estimated
reference mean,
$T_t = w\sum_{j=1}^d(\overline X_{w,j}-\hat\mu_{0,j})^2/\hat\sigma_{0,j}^2$,
with $\hat\mu_0,\hat\sigma_0^2$ estimated from the burn-in segment (a
diagonal rather than full covariance, since a full $d\times d$ estimate
from only $w$ burn-in observations is singular whenever $d>w$). By the
CLT, $T_t$ is asymptotically $\chi^2_d$ under $H_0$ -- unlike A, this
\emph{is} a known, closed-form null, so the threshold comes directly from
simulating a CUSUM fed by i.i.d.\ $\chi^2_d$ increments with slack $k=d$
(the null mean), no data of any kind required for calibration itself:
$g_t=\max(0,\,g_{t-1}+T_t-k)$, alarm at the first $t$ with $g_t\ge h^*$.
This is deliberately the simplest thing a practitioner reaches for first
-- and, being asymptotic rather than exact, its calibration is only as
good as the CLT approximation at the chosen window length $w$, which
Table~\ref{tab:baselines} shows degrading badly whenever $p_0$ itself is
far from Gaussian at that $w$ (the two-component mixtures here).

\subsection*{Monte Carlo protocol}
For each experiment, all four detectors are calibrated
\emph{once}: a single burn-in sample of $w$ observations from $p_0$ is
drawn, pilot pools are drawn from $p_1$ exactly as in
Appendix~\ref{app:pilot-calibration}, and each method's threshold is set
from its own null model at the shared budget $P_{\rm FA}\le5\%$ -- Monte
Carlo simulation under the exact $\mathcal N(0,I)$ null for ours/B,
bootstrap of the burn-in for A, and the closed-form $\chi^2_d$-CUSUM
construction for C. This mirrors how a threshold is actually used in
practice: set once from a false-alarm budget, then applied to every
future monitoring episode, not re-derived per episode. Holding the
threshold fixed, $1{,}000$ independent test series of length $L$ are then
drawn, each with a changepoint $\tau$ sampled uniformly at random in
$[w,L)$ -- i.e.\ never before the burn-in window has closed, consistent
with the only assumption Section~\ref{sec:problem} licenses -- and every
method's fixed threshold is run against every series.

\end{document}